\documentclass[letterpaper, 10 pt, journal, twoside]{IEEEtran}
\usepackage{amsmath,amsfonts}
\usepackage{algorithmic}
\usepackage{array}
\usepackage[caption=false,font=normalsize,labelfont=sf,textfont=sf]{subfig}
\usepackage{textcomp}
\usepackage{stfloats}
\usepackage{url}
\usepackage{verbatim}
\usepackage{graphicx}
\usepackage{cite}
\usepackage[ruled,vlined,linesnumbered]{algorithm2e}
\SetCommentSty{emph}
\usepackage{multirow}
\usepackage{graphicx} 
\usepackage{makecell}   
\usepackage{amsmath}    
\usepackage{url}
\usepackage{xcolor}
\definecolor{winkcolor}{RGB}{0,0,150}     
\definecolor{kTPGucolor}{RGB}{150,0,0}    
\usepackage{amssymb,amsmath,comment}
\usepackage{amsthm}
\usepackage{amsfonts}
\usepackage[capitalize]{cleveref}
\usepackage{booktabs}
\newenvironment{proofsketch}{\begin{proof}[Proof sketch]}{\end{proof}}
\Crefname{equation}{Eq.}{Eqs.}
\Crefname{figure}{Fig.}{Figs.}
\Crefname{tabular}{Tab.}{Tabs.}
\newtheorem{theorem}{Theorem}

\newcommand{\mymethod}{WinkTPG\xspace}
\newcommand{\planner}{kTPG\xspace}
\newcommand{\uncertainPlanner}{kTPGu\xspace}
\newcommand{\repolink}{https://github.com/JingtianYan/IEEE-T-ASE-WinkTPG}

\begin{document}

\title{WinkTPG: An Execution Framework for Multi-Agent Path Finding Using Temporal Reasoning}


\author {
    Jingtian Yan,~\IEEEmembership{Graduate Student Member,~IEEE},
    Stephen F. Smith,~\IEEEmembership{Member,~IEEE},\\
    Jiaoyang Li,~\IEEEmembership{Member,~IEEE}
\thanks{Manuscript received: December, 24, 2025; Revised March, 20, 2026; Accepted April, 19, 2026. This paper was recommended for publication by Editor Paolo Rocco upon evaluation of the Associate Editor and Reviewers' comments. This work was supported by the National Science Foundation under Grants \#$2328671$ and \#$2441629$, as well as a gift from Amazon. (Corresponding author: Jingtian Yan.)}
\thanks{$^{1}$J. Yan, S. Smith, and J. Li are with Robotics Institute, Carnegie Mellon University, Pittsburgh, PA 15213, USA {(email: \{jingtiay, sfs, jiaoyanl\}@cs.cmu.edu)}.}
\thanks{Digital Object Identifier (DOI): see top of this page.}
}

\markboth{IEEE TRANSACTIONS ON AUTOMATION SCIENCE AND ENGINEERING. PREPRINT VERSION. ACCEPTED APRIL, 2026}%
{Yan \MakeLowercase{\textit{et al.}}: WinkTPG: An Execution Framework for Multi-Agent Path Finding Using Temporal Reasoning}


\maketitle

\begin{abstract}
Planning collision-free paths for a large group of agents is a challenging problem in many real-world applications.
While recent advances in Multi-Agent Path Finding (MAPF) have shown promising progress, standard MAPF planners continue to rely on simplified kinodynamic models, preventing agents from directly following the generated MAPF plan.
To bridge this gap, we propose kinodynamic Temporal Plan Graph planning (\planner), a multi-agent speed optimization algorithm that efficiently refines a MAPF plan into a set of kinodynamically feasible speed profiles. We further incorporate execution timing uncertainty models and provide deterministic guarantees under bounded uncertainty models and probabilistic guarantees under stochastic models.
Building on \planner, we propose Windowed \planner (\mymethod), a MAPF execution framework that incrementally refines MAPF plans using a window-based mechanism, dynamically incorporating agent information during execution to reduce uncertainty.
Experiments show that \mymethod can generate speed profiles for up to 1,000 agents within 1 second and improves solution quality by up to 51.7\% over existing MAPF execution methods.
We further validate \mymethod in high-fidelity physics-based simulation and on real-world robots.
\end{abstract}

\def\abstractname{Note to Practitioners}
\begin{abstract}
The motivation of this article originates from the need to execute large-scale multi-agent path-planning solutions reliably in practical applications such as warehouse logistics, industrial material transport, and factory or airport automation.
Although discrete MAPF planners can
generate conflict-free paths, their solutions are often not directly executable by real robots due to kinodynamic constraints and execution timing uncertainties. This article develops WinkTPG, a windowed execution framework that converts MAPF plans into safe and kinodynamically feasible speed profiles while maintaining precedence relations implied by the discrete plan. WinkTPG further employs receding-horizon replanning to correct timing deviations caused by disturbances or sensing noise, making it suitable for practitioners deploying multi-robot fleets in constrained, time-critical environments.
\end{abstract}

\begin{IEEEkeywords}
Multi-Agent Path Finding, Multi-Robot System, Multi-Robot Execution.
\end{IEEEkeywords}


\section{Introduction}
\IEEEPARstart{C}{ollision-free} coordination of multiple robotic agents operating in the same physical space is an important task in numerous real-world applications, including automated warehouses~\cite{wurman2008coordinating,ma2017feasibility}, traffic intersections~\cite{li2023intersection}, and airports~\cite{morris2016planning}.
Multi-Agent Path Finding (MAPF) methods offer significant advantages for solving this problem, such as scalability and (bounded) optimality guarantees.
However, applying MAPF in these domains remains challenging:
most MAPF methods ignore real-world factors like kinodynamics (e.g., speed and acceleration limits) and execution timing uncertainty (i.e., variability in traversal times due to controller noise, actuation latency, and slippage), making it impractical for agents to strictly follow the planned paths.

One promising solution is to postprocess the plan generated by MAPF methods. Honig et al.~\cite{honig2016multi} introduce the Temporal Plan Graph (TPG), which refines the speed profiles of agents to meet their speed limits while preserving the passing orders of agents at different locations.
While TPG ensures collision-free and deadlock-free movement, it cannot handle kinodynamic constraints beyond speed limits. That is, it may generate speed profiles with infinitely large accelerations, making them infeasible for agents to follow.
Two extensions of TPG overcome this problem: the Kinodynamic Network (KDN)~\cite{zhang2021temporal} and the Action Dependency Graph (ADG)~\cite{honig2019warehouse}. KDN incorporates kinodynamic constraints into TPG but suffers from limited scalability and cannot handle execution timing uncertainty.
In contrast, ADG is a scalable execution framework that transforms a MAPF plan into a dependency graph and coordinates agents through event-triggered execution. However, each agent in an ADG plans its own speed profile independently, without accounting for the planned or predicted motion of other agents.
Instead of jointly coordinating speed profiles, each agent reacts only after its dependencies are formally resolved, ignoring the future motion of other agents even when it could be anticipated. Consequently, ADG often under-utilizes available temporal flexibility, causing unnecessary deceleration and reducing overall execution efficiency even when globally coordinated motion would be feasible.

To overcome this limitation, we propose \textbf{Win}dowed \textbf{K}inodynamic \textbf{T}emporal \textbf{P}lan \textbf{G}raph planning (\mymethod), an efficient execution framework for MAPF.
Unlike event-driven methods such as ADG, which make execution decisions for each agent independently without global coordination, \mymethod jointly plans speed profiles for all agents under the temporal structure induced by the TPG.
Given a set of collision-free paths, \mymethod first constructs a TPG to encode the passing orders of agents at shared locations.
It then applies the proposed kinodynamic planner, \planner, to generate coordinated speed profiles for all agents. The key idea of \planner is to use \emph{reserved intervals} to allocate shared locations to competing agents over time. These intervals specify when agents are allowed to traverse particular locations, based on kinodynamic reachability and the precedence structure of the TPG. By propagating and updating reserved intervals across agents, \planner captures global temporal coupling among agents rather than making purely local, event-triggered decisions. This enables \planner to optimize the motions of agents in a coordinated manner and to reduce the conservative behavior commonly incurred by ADG.
To improve robustness under modeled execution timing uncertainty, we incorporate a safety-margin mechanism into \planner.
Under bounded execution timing uncertainty models, our method guarantees deterministic precedence satisfaction and collision-free execution; under a stochastic uncertainty model, it yields probabilistic guarantees.

To enhance robustness and efficiency during execution, \mymethod couples \planner with the use of a windowed execution mechanism. This mechanism uses updated information during execution to reduce uncertainty and improve overall performance through dynamic replanning.
Theoretically, we prove that \planner and \mymethod are complete. Empirically, we evaluate them on the standard MAPF benchmarks with two robot models and three execution timing uncertainty models. \mymethod shows up to 51.7\% improvement in terms of solution quality compared to ADG.
Moreover, \mymethod generates speed profiles for 1,000 agents within 1 second, whereas KDN may take over 300 seconds to generate solutions for only 200 agents.
Finally, we demonstrate the real-world feasibility of \mymethod by validating it in a high-fidelity physics-based simulation environment and on real-world robots.

\cref{fig:method_flow} provides an overview of the proposed framework and also serves as a guide to the organization of this paper.
Starting from a discrete MAPF plan, we first
formulate the multi-agent execution problem (MAEP) and introduce \planner in \cref{sec:ktpg}.
Then, in~\cref{sec:ktpgu} we present the uncertainty-aware
extension of MAEP and discuss how execution timing uncertainty is incorporated into the
planning process. Building on this formulation, we next describe \mymethod, a
windowed execution strategy that enables online replanning using updated execution information in~\cref{sec:winktpg}. Theoretical and Empirical results are presented in Sections \ref{sec:theory} and \ref{sec:exps}, respectively.

\section{Preliminaries}
Before describing our method, we first define the MAPF problem and the temporal plan graph that encodes the MAPF plan, and review related works on executing MAPF plans.

\begin{figure*}
    \centering
    \includegraphics[width=\linewidth]{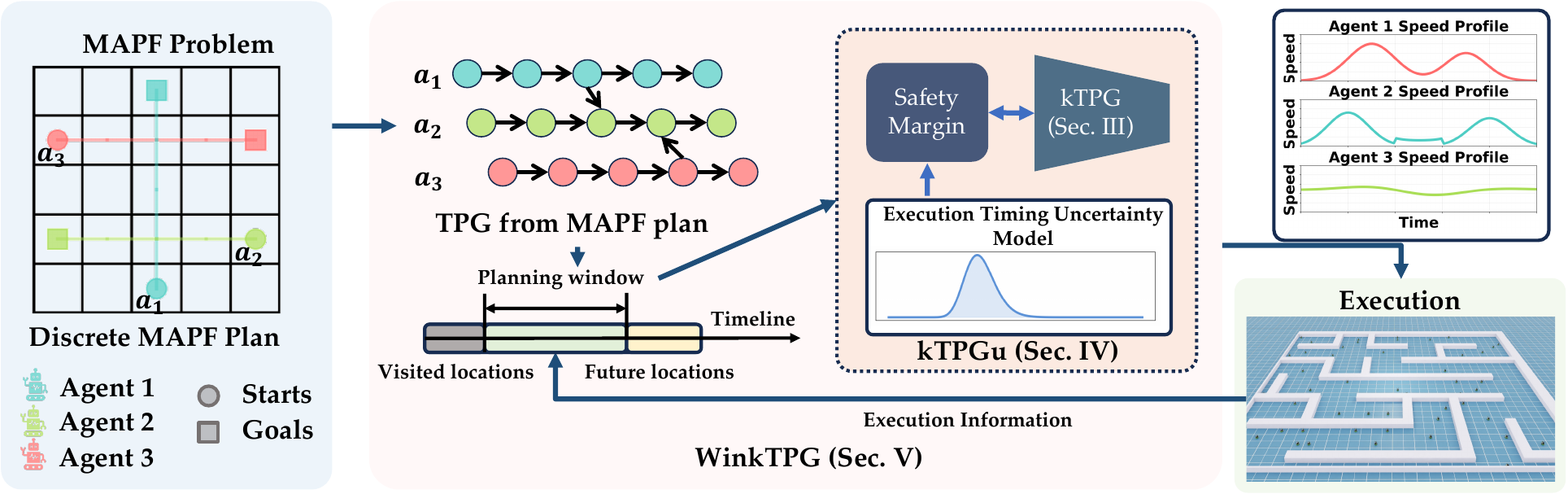}
    \caption{Overview of \mymethod. \mymethod produces robust speed profiles for all agents given a MAPF plan. A discrete MAPF plan is first converted into a Temporal Planning Graph (TPG). Kinodynamic constraints are incorporated to form kTPG (Sec. III), and execution uncertainty is modeled in kTPGu (Sec. IV).
During execution, WinkTPG (Sec. V) updates temporal windows based on real-time execution feedback, enabling adaptive and less conservative execution.}
    \label{fig:method_flow}
\end{figure*}

\begin{figure}
    \centering
    \includegraphics[width=\linewidth]{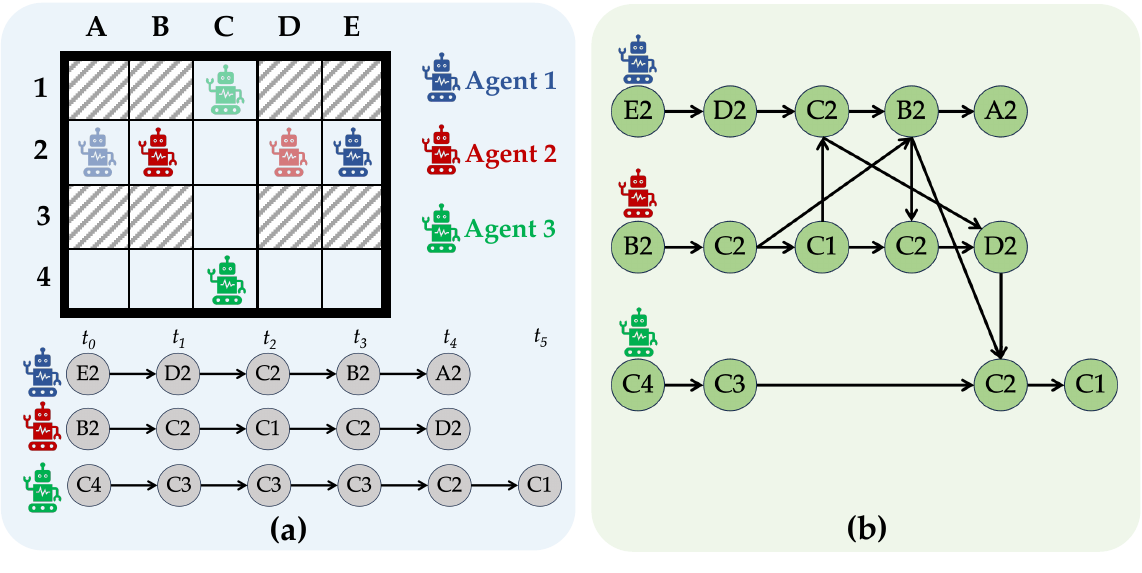}
    \caption{Example of a MAPF instance, a MAPF plan, and a TPG.}
    \label{fig:example_tpg}
\end{figure}

\subsection{MAPF Problem}\label{sec:mapf_problem}
As shown in~\cref{fig:example_tpg} (a), a MAPF problem~\cite{Stern2019benchmark} contains an undirected graph $\mathcal{G}_U = (\mathcal{V}_U, \mathcal{E}_U)$ and $I$ agents $\mathcal{A} = \{a_1, \dots, a_I\}$, each with a  start $q_i^s \in \mathcal{V}_U$ and goal $q_i^g \in \mathcal{V}_U$.
Time is discretized, and at each timestep, an agent can either move to an adjacent vertex or remain at its current vertex.
A MAPF plan is a set of collision-free paths $\mathcal{P} = \{p_1, \dots, p_I\}$, where the path of $a_i$ is a sequence of time-annotated vertices: $p_i = \{(q_i^{0}, t_0), \dots, (q_i^{z^i}, t_{z^i})\}$. Here, $q_i^{0} = q_i^s, q_i^{z^i}= q_i^g$, and for all $k \in \{1, \dots, z^i\}, q_i^{k} \in \mathcal{V}_U$ and $(q_i^{{k-1}}, q_i^{{k}}) \in \mathcal{E}_U$.
$\mathcal{P}$ is defined as collision-free if it contains
\emph{no vertex collisions}: agents must not occupy the same vertex at the same timestep, and
\emph{no edge collisions}: agents must not swap vertices at the same timestep.
To avoid confusion with the TPG, we henceforth refer to vertices in $\mathcal{V}_U$ as \emph{locations}.
To be compatible with the TPG representation, we assume that the MAPF plan is also cycle-conflict-free (i.e., no subset of agents swap each other’s locations at the same timestep~\cite{Stern2019benchmark}).

\subsection{Temporal Plan Graph (TPG)}
A Temporal Plan Graph (TPG), as shown in~\cref{fig:example_tpg} (b), is a directed acyclic graph $\mathcal{G}=(\mathcal{V},\mathcal{E}_1,\mathcal{E}_2)$ that represents a MAPF plan $\mathcal{P}$ by capturing precedence relationships for visiting locations.
The vertex set $\mathcal{V} = \{v_i^k : k \in [0, z^i], a_i \in \mathcal{A}\}$ comprises all locations to be visited sequentially by agent $a_i$ (i.e., $loc(v_i^k) = q_i^k$).
The edge set $\mathcal{E}_1$ includes \emph{Type-1 edges}, which capture sequential precedence within the path of an agent.
A Type-1 edge $(v_i^k, v_i^{k+1})$ ensures that $a_i$ visits $v_i^k$ before $v_i^{k+1}$.
The edge set $\mathcal{E}_2$ consists of \emph{Type-2 edges}, which enforce the order in which agents visit shared locations. For any two vertices $v_j^s$ and $v_i^k$ such that $loc(v_j^s) = loc(v_i^k)$, a Type-2 edge $(v_j^{s+1}, v_i^k)$ ensures that $a_i$ can enter $v_i^k$ only after $a_j$ leaves $v_j^s$ and reaches $v_j^{s+1}$.

\subsection{Related Works}
MAPF methods compute collision-free paths for large groups of agents.
Leading methods, such as PBS~\cite{ma2019searching}, PIBT~\cite{okumura2022priority}, and EECBS~\cite{li2021eecbs}, can find collision-free paths for thousands of agents.
However, these methods assume all agents move synchronously, which is unrealistic in real-world scenarios due to their reliance on simplified, discretized kinodynamic models.
This limitation is also reflected in prior studies that demonstrate the execution of MAPF plans on physical multi-robot platforms~\cite{bartak2019multi,yan2025advancing,saccon2022comparing,lehoux2024multi,wen2022cl}, highlighting the practical importance of bridging the gap between discrete MAPF plans and continuous real-world execution.
To address this issue, some methods~\cite{andreychuk2021improving,yan2024PSB,moldagalieva2024db} integrate more accurate kinodynamic models during planning.
However, these methods often face scalability challenges due to the increased computational complexity introduced by the kinodynamic models.
Another group of methods~\cite{honig2016multi,zhang2021temporal} first generates a MAPF plan with a simplified kinodynamic model and subsequently derives speed profiles with the accurate kinodynamic model based on these paths. 
For instance, KDN~\cite{zhang2021temporal} employs a TPG to encode the precedence constraints defined in the MAPF plan. It then uses a Mixed-Integer Linear Programming (MILP) solver to find kinodynamically feasible speed profiles for agents based on the TPG.

To handle execution timing uncertainty, some methods~\cite{atzmon2020robust,peltzer2020stt} account for this uncertainty during planning.
These methods first build a model of the execution timing uncertainty and then plan collision-free paths, ensuring that agents remain collision-free even when time uncertainties occur as described by the models.
While effective in some cases, these methods often lead to poor performance due to being overly conservative during the planning stage, as they cannot access the actual delays that occur during execution. Additionally, these methods are often limited in scalability.

Introducing an execution framework to MAPF is a promising method that enables scalable coordination while satisfying agent kinodynamic constraints and ensuring robust execution.
Action Dependency Graph (ADG)~\cite{honig2019warehouse} uses a TPG to ensure the passing order and defines an execution framework for speed profile generation.
Specifically, it categorizes the vertices in the graph into three statuses: \emph{staged}, \emph{enqueued}, and \emph{finished}. A vertex is marked as \emph{staged} if the agent cannot move to it yet.
It transitions to \emph{enqueued} when the agent can safely move here. This happens only if it has no preceding vertices or if the following conditions are met: (1) its preceding vertex connected by a Type-1 edge is \emph{enqueued} or \emph{finished}, and (2) all preceding vertices connected by Type-2 edges are \emph{finished}.
Once a vertex transitions to \emph{enqueued}, ADG generates a speed profile for the agent to reach the vertex.
A vertex is marked as \emph{finished} when the agent has reached it.
ADG dynamically updates vertex statuses during execution and provides robustness guarantees under execution timing uncertainty.

Although ADG provides robustness guarantees, it does not explicitly account for the speed profiles and execution timing of other agents when reasoning dependencies. 
Because dependency satisfaction is determined solely by predecessor completion events, ADG cannot reason about when a shared location will become temporally feasible based on ongoing execution of other agents. 
As a result, it must adopt a conservative strategy in which an agent waits until all required predecessors have formally finished, even when their planned speed profiles already imply conflict-free traversal.
This conservative strategy can lead to unnecessary deceleration.
As shown in~\cref{fig:intro_idea}, R2 (blue) must wait until R1 (green) leaves cell B2 before it can proceed. Although it is evident early on that R1 will leave B2 before R2 arrives, ADG does not allow R2 to commit to the optimal speed profile until the dependency is officially satisfied.
This results in R2 decelerating prematurely (top chart in (c)), compared to the optimal case (bottom chart) where it could have maintained maximum speed.
The problem is worsened by potential communication delays in real-world settings, where agents rely on frequent updates from a central server. 
Consequently, this overly conservative strategy can result in poor performance in many scenarios.

\begin{figure}
    \centering
    \includegraphics[width=.95\linewidth]{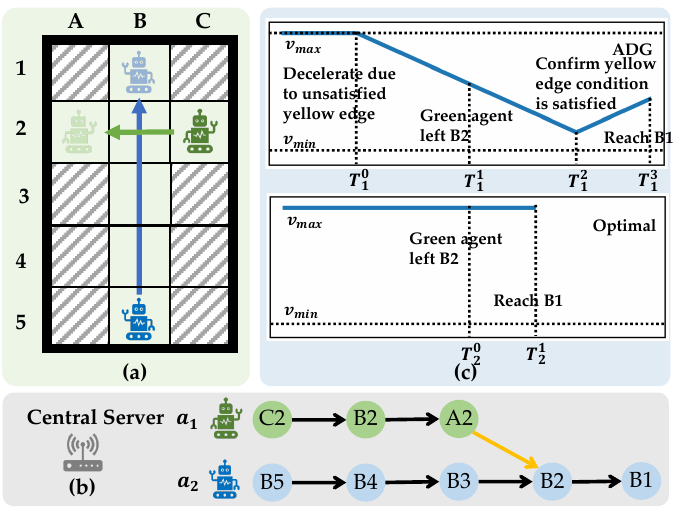}
    \caption{
    (a) MAPF instance with a green agent (R1) moving from C2 to A2, and a blue agent (R2) moving from B5 to B1.
    Both agents have an initial speed of $v_{max}$. 
    (b) TPG generated by the central server, where the yellow edge indicates that R2 can only move to B2 after R1 has left it.
    (c) ADG-generated and optimal speed profiles for R2.}
    \label{fig:intro_idea}
\end{figure}

\section{Kinodynamic TPG (kTPG)} \label{sec:ktpg}
\cref{sec:maep} below specifies the problem addressed in this section. 
In~\Cref{sec:ktpgu,sec:winktpg}, we will extend this definition to account for execution timing uncertainty during agent movement.

\subsection{Multi-Agent Execution Problem (MAEP)} \label{sec:maep}
Given a MAPF plan, the objective of MAEP~\cite{zhang2021temporal} is to generate a set of speed profiles for all agents to follow during execution. These profiles must satisfy the kinodynamic constraints imposed by the agents' physical models, respect the precedence relationships defined in the MAPF plan, and preserve collision-free movement. The objective is to minimize the total time for all agents to reach their goals.

Let \(A = \{a_1, \dots, a_I\}\) denote the set of agents defined in~\cref{sec:mapf_problem}, now operating in a continuous workspace \(\mathcal{W} \subset \mathbb{R}^d\), with \(d \in \{2,3\}\). Given a MAPF plan \(\mathcal{P} = \{p_1, \dots, p_I\}\), we reinterpret each discrete path \(p_i\) as a sequence of waypoints in continuous space.
Specifically, for agent \(a_i\), its path is given by $p_i = \big\{(q_i^0, t_i^0), (q_i^1, t_i^1), \dots, (q_i^{z_i}, t_i^{z_i})\big\}$, where each \(q_i^k\) corresponds to a vertex in the MAPF graph $\mathcal{G}_U = (\mathcal{V}_U, \mathcal{E}_U)$ and represents a location in $\mathcal{W}$.
During execution, each agent is required to spatially follow its assigned path,
i.e., visit the waypoints in order and move along straight-line segments between consecutive waypoints.
To execute its path in continuous space, each agent $a_i$ follows a continuous
trajectory $x_i(\cdot)$ that must satisfy all kinodynamic and safety constraints imposed by its physical model.
We capture these limits by an admissible trajectory set:
\[
\mathcal{X}_i
=
\Big\{
x_i : [0,\infty) \to \mathbb{R}^{d_i}
\;\Big|\;
C_i\big(x_i(\cdot)\big) \le 0
\Big\},
\]
where $d_i$ is the dimension of the state space of agent $a_i$, and \( C_i(\cdot) \) is a constraint operator encoding the kinodynamic feasibility requirements of \(a_i\), including bounds on velocity, acceleration, jerk, actuator capabilities, and other safety constraints. The operator \(C_i(\cdot)\) may represent pointwise-in-time constraints (e.g., \( \|\dot{x}_i(t)\| \le v_{\max} \), \( \|\ddot{x}_i(t)\| \le a_{\max} \)) as well as integral or state-dependent conditions. A trajectory is considered kinodynamically feasible if and only if it belongs to \( \mathcal{X}_i \).
We assume that waiting at waypoints is always kinodynamically feasible under \(C_i(\cdot)\).
We assume that the spatial paths produced by the MAPF planner are executable, in the sense that for each agent $a_i$ there exists at least one continuous trajectory $x_i(\cdot)\in\mathcal{X}_i$ that follows the assigned path $p_i$.

Although agents move in continuous space, potential inter-agent collisions can only occur at a finite set of shared waypoints defined by the MAPF plan.
We therefore use the underlying MAPF graph $\mathcal{G}_U$ to reason about collision avoidance during execution.
This graph abstraction allows collision avoidance to be enforced through temporal constraints at waypoints, rather than continuous-space reasoning.
We define the \emph{reach time} $t_i^r(q_i^k)$ as the time at which $a_i$ arrives at waypoint $q_i^k$, and
the corresponding \emph{leave time}
$t_i^l(q_i^k) = t_i^r(q_i^{k+1})$ as the time at which it departs toward the next
waypoint. We say $a_i$ occupies waypoint $q_i^k$ during $[t_i^r(q_i^k), t_i^l(q_i^{k})]$.

We define a trajectory set $\{x_i(t)\}_{i=1}^I$ as \emph{valid} for $\mathcal{P}$ if it satisfies:

\begin{enumerate}
    \item \textbf{Precedence Constraints:}
    For every pair of consecutive waypoints assigned to an agent, the execution must respect the
    ordering implied by the MAPF plan:
    \[
        t^r_i(q_i^k) < t^r_i(q_i^{k+1}), 
        \qquad \forall i, \forall k,
    \]
        If agents $a_i$ and $a_j$ both visit $q$, and $a_i$ is assigned to visit $q$ before $a_j$, then
    \[
        t^l_i(q) \le t^r_j(q).
    \]

    \item \textbf{Kinodynamic Feasibility:}
    \[
    x_i(\cdot) \in \mathcal{X}_i, \qquad \forall i,\ \forall t.
    \]

\end{enumerate}
Note that the precedence constraint is sufficient to prevent both vertex and edge collisions.
Specifically, the constraint $t^l_i(q) \le t^r_j(q)$ ensures that the occupancy intervals of agents at any shared location do not overlap, thereby preventing vertex collisions. 
Furthermore, an edge collision would require two agents to traverse the same undirected edge in opposite directions over overlapping time intervals. 
Such a situation necessarily implies overlapping occupancy of at least one shared endpoint. 
Since the MAPF plan is collision-free, and since the execution preserves waypoint order while prohibiting overlapping occupancy at any vertex, opposite-direction traversal cannot occur.

\noindent\textbf{Objective.}
Given the MAPF plan $\mathcal{P}$ and the kinodynamic models $\{\mathcal{X}_i\}$, the
Multi-Agent Execution Problem is
\[
\begin{aligned}
\min_{\{x_i(\cdot)\}} \quad 
    & \sum_{i=1}^I t^r_i(q_i^{z_i}) \\
    \text{s.t. } & 
    \{x_i(\cdot)\} \text{ is valid for } \mathcal{P}.
\end{aligned}
\]

\subsection{\planner Overview}
The high-level idea of \planner is to encode the precedence constraints in a MAPF plan into a temporal plan graph (TPG)~$\mathcal{G}$ and associate each vertex with a \emph{reserved interval}. 
Beyond encoding precedence relations, \planner synthesizes execution timing information from the speed profiles of multiple agents together with kinodynamic reachability constraints into these reserved intervals. 
Each reserved interval therefore specifies when an agent can arrive at a location without conflicting with others or violating its kinodynamic constraints.
\planner generates speed profiles for individual agents independently, ensuring that each agent visits the vertices along its path within the corresponding reserved intervals. 

All reserved intervals are initialized to $[0,\infty)$, and \planner iteratively refines both the intervals and the associated speed profiles. 
When a speed profile is fixed for one agent, the induced occupancy timing constrains the feasible arrival windows of dependent vertices, and these constraints are propagated through $\mathcal{G}$. 
The refinement process continues until the reserved intervals of vertices corresponding to the same location do not overlap, ensuring collision-free execution.
We say that a set of speed profiles \emph{satisfies} the edges in $\mathcal{G}$ if the occupancy intervals induced by the speed profiles do not violate the precedence constraints encoded by the edges. 
Satisfaction of all edges in the TPG guarantees that the precedence relations defined in the MAPF plan are preserved.

\cref{alg:ktpg} summarizes the overall procedure of \planner.
Using the example in~\cref{fig:method_overview}, we illustrate a single iteration of this process.
Each iteration begins by selecting an agent to refine its speed profile, following the selection rule in~\cref{method:solve_order}; in this case, $a_2$ is selected.
We define a Type-2 edge as \emph{satisfied} if the reserved intervals at its associated vertices do not overlap and the visiting precedence is preserved. Otherwise, it is \emph{conflicting}. All Type-2 edges in our example are conflicting in this iteration.
We then generate a speed profile for $a_2$ from its first to last \emph{unlocked} vertex, i.e., those vertices without any incoming conflicting edges, using the method in~\cref{sec:SSP}. These vertices are shown as dark green circles in \Cref{fig:method_overview} (b).
Next, we identify the conflicting edges originating from these unlocked vertices (i.e., the red edge in \Cref{fig:method_overview} (b)).
We use the generated speed profile to derive a split time for the target vertex (here, C2) and update the reserved intervals for both agents accordingly so that the intervals no longer overlap (see~\cref{method:conflict_edge}).
This update allows us to satisfy a conflicting edge and unlock additional vertices of other agents, as shown in \Cref{fig:method_overview} (c).
In the next iteration, we use the newly unlocked vertices to satisfy more Type-2 edges by selecting another agent and updating its speed profile.
This process repeats until all Type-2 edges are satisfied, at which point we have complete, collision-free speed profiles for all agents.

\begin{figure*}
    \centering
    \includegraphics[width=\linewidth]{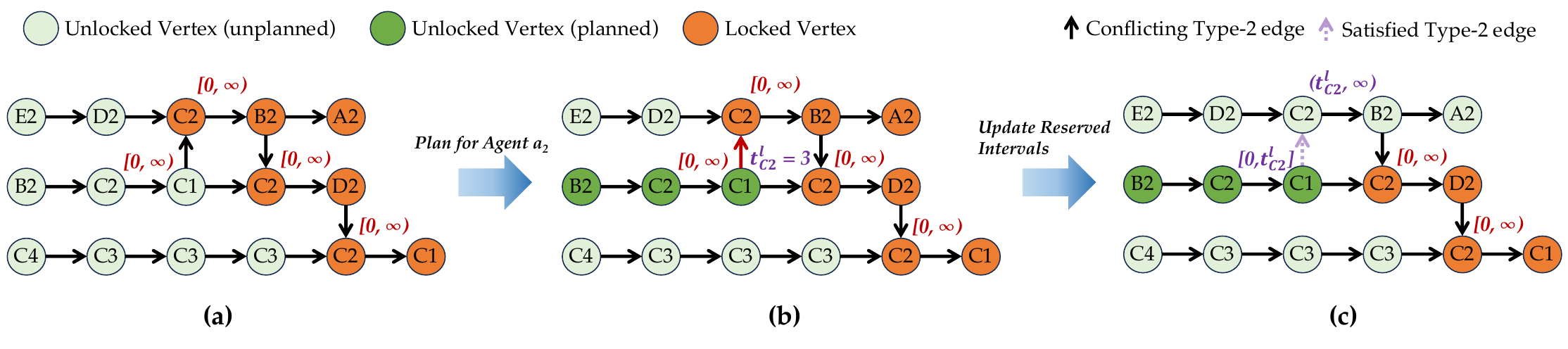}
    \caption{Updating reserved intervals to satisfy the conflicting Type-2 edge from C1 to C2.}
    \label{fig:method_overview}
\end{figure*}

\begin{figure}
    \centering
    \includegraphics[width=\linewidth]{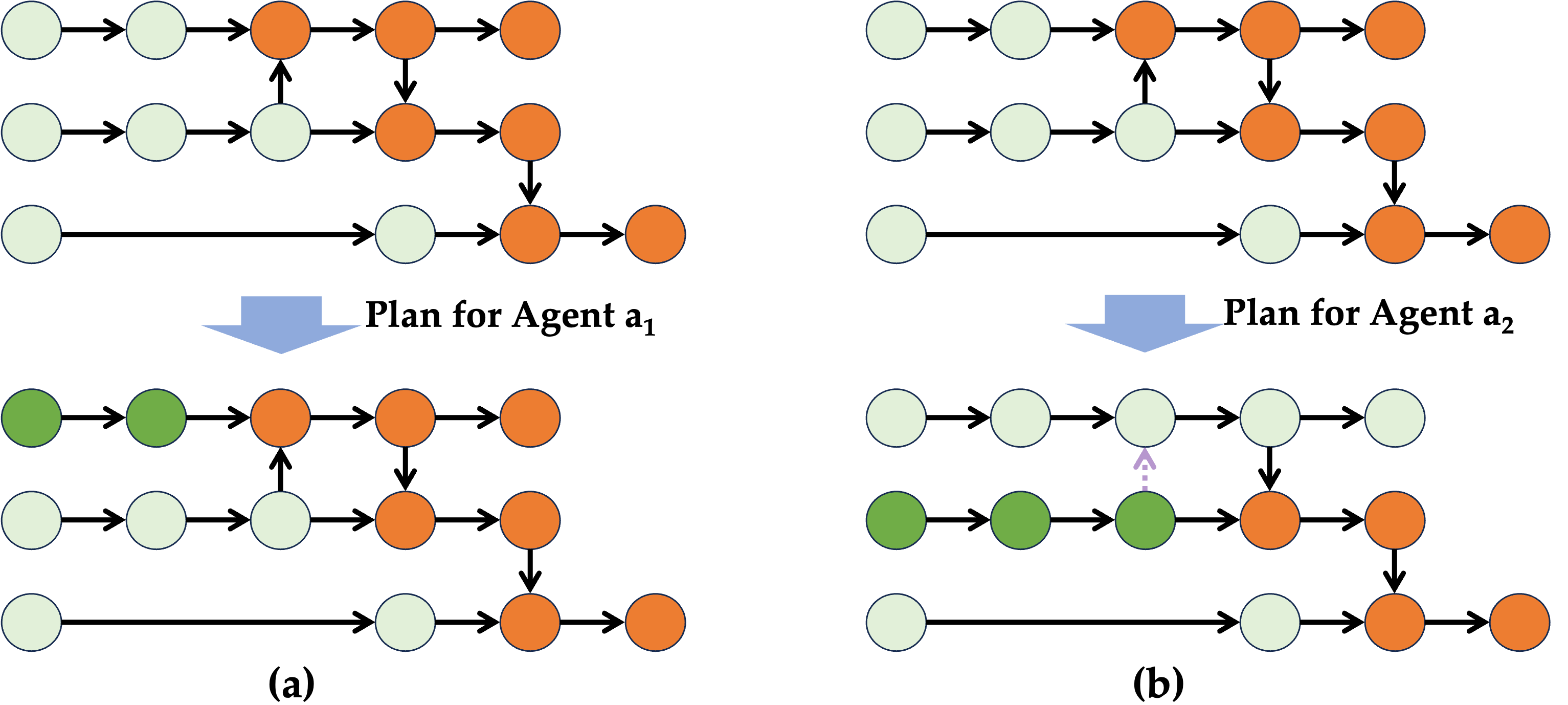}
    \caption{\planner with agent selection: (a) $a_1$ first, (b) $a_2$ first.}
    \label{fig:method_solving_order}
\end{figure}

\subsection{Selecting Agent} \label{method:solve_order}
The order in which agents are selected for replanning is important.
As shown in~\cref{fig:method_solving_order}, if we plan for $a_1$ first, no additional vertices are unlocked in that iteration. In contrast, planning for $a_2$ first results in three vertices being unlocked.
Thus, the selection order directly affects the number of vertices unlocked per iteration, which in turn influences the total number of iterations required to unlock all vertices.
Motivated by this observation, \planner adopts the following strategy:
At each iteration, we select the agent that can satisfy the greatest number of conflicting Type-2 edges originating from its currently unlocked vertices.


\begin{algorithm}[t]
\caption{\planner}
\label{alg:ktpg}
\KwIn{MAPF plan $P = \{p_i\}_{i=1}^I$, kinodynamic models $\{\mathcal{X}_i\}$}
\KwOut{Speed profiles $\{x_i(\cdot)\}_{i=1}^I$}

Construct TPG $\mathcal{G}$ from $P$ and initialize reserved intervals $\mathcal{R}$;\\
Initialize $\mathcal{C}\leftarrow\emptyset$ \tcp*{$\mathcal{C}$: agents with complete speed profiles to their goals}
Mark all Type-2 edges in $\mathcal{G}$ as unsatisfied;

\While{$|\mathcal{C}| < I$}{
    Select an agent $a_i$ to refine;\\
    Determine the set of unlocked vertices on $a_i$'s path;\\
    Compute a kinodynamically feasible speed profile $x_i(\cdot)$ on the unlocked
    vertices, subject to $\mathcal{X}_i$ and $\mathcal{R}$;\\
    Update the reserved intervals $\mathcal{R}$;\\
    Mark any Type-2 edges that have become satisfied as \emph{satisfied};\\
    \If{$x_i(\cdot)$ is computed up to the goal}{
        $\mathcal{C}\leftarrow \mathcal{C}\cup\{i\}$;
    }
}

\Return{$\{x_i(\cdot)\}_{i=1}^I$}
\end{algorithm}

\subsection{Updating Speed Profile}\label{sec:SSP}
After selecting an agent, we compute a speed profile for its 1-D path that (i) visits each unlocked vertex within its reserved interval, (ii) starts and ends with zero speed, and (iii) satisfies kinodynamic constraints.
Because the unlocked vertices have no incoming conflicting edges, any speed profile that respects the reserved intervals is collision-free with respect to already-fixed agents, and can be used to update the reserved intervals.

We formulate this step as a single-agent planning problem in \emph{time--distance} space along the fixed geometric path: each vertex induces a safe time window in which the agent may occupy the corresponding shared location, and each edge traversal must respect the agent's kinodynamic limits.
This problem can be addressed using MILP-based trajectory optimization methods~\cite{lipp2014minimum,liu2017speed} or search-based motion primitives methods~\cite{honig2022db,ali2023safe}.
Empirically, we achieve this by using a 1-D adaptation of SIPP-IP~\cite{ali2023safe}, an optimal algorithm that performs an A*-type search along the vertices using a set of predefined motion primitives.
Given the predefined motion primitives, SIPP-IP is complete and optimal, i.e., it provably finds the speed profile that minimizes the reach time at the last unlocked vertex if one exists and returns failure otherwise.
Importantly, the proposed framework is \emph{not} tied to a particular speed-profile generation method. It requires only a single-agent planner capable of producing a speed profile that respects the reserved time intervals and satisfies the kinodynamic constraints encoded in $\mathcal{X}_i$. Alternative kinodynamic models, such as jerk-limited $S$-curve profiles, can therefore be incorporated by modifying the underlying primitive set or replacing the trajectory optimizer, without altering the temporal reasoning mechanism or the reserved-interval update rules of \planner.

\subsection{Satisfying Conflicting Edges}\label{method:conflict_edge}
After obtaining the speed profile for $a_i$, we update the reserved intervals to satisfy conflicting edges.
Specifically, for each conflicting edge $(v_i^{k+1}, v_j^s)$ originating from an unlocked vertex $v_i^{k+1}$ of $a_i$, 
we use the leave time of $a_i$ at the shared location $loc(v_i^k) = loc(v_j^s)$, denoted as $t^l$, as the split time to divide the time interval at the shared location into $[0, t^l]$ and $(t^l, \infty)$. We enforce that $a_i$ and $a_j$ occupy the shared location within these intervals respectively.
Concretely, we update the reserved intervals for $v_i^k$ and $v_j^s$ by intersecting them with the corresponding segments.

\cref{fig:method_overview} shows an example.
Once the speed profile for agent $a_2$ is obtained, we find all conflicting edges originating from unlocked vertices of $a_2$, which is the first Type-2 edge from $a_2$ to $a_1$, with shared location $C_2$. From the speed profile of $a_2$, we use the leave time $t_2^l(C_2)$ as the split time. 
We divide the time interval at $C_2$ into $[0, t_2^l(C_2)]$ and $(t_2^l(C_2), \infty)$.
The reserved intervals of $a_1$ and $a_2$ are updated by intersecting their current reserved intervals with these segments, respectively.
By assigning these intervals, the Type-2 edge is marked as satisfied.
Any speed profile generated for $a_2$ in subsequent iterations must adhere to this updated interval.

We use the leave time as the split time to ensure that a feasible speed profile always exists when replanning for $a_i$ in future iterations.
Concretely, since, in the current iteration, the reserved intervals at the last unlocked vertex and all following locked vertices have an infinite upper bound, the next time when we replan $a_i$, $a_i$ can follow the speed profile generated in the current iteration, wait at the last unlocked vertex as needed, and continue without violating constraints.

\section{kTPG with Uncertainty}\label{sec:ktpgu}

In the previous section, execution was assumed deterministic. In realistic deployments, agents experience (execution timing) uncertainty due to controller noise, actuation latency, slippage, and other execution-level perturbations. To account for such temporal uncertainty, we extend the Multi-Agent Execution Problem (MAEP) to include uncertainty models and derive model-conditional safety margins that ensure deterministic or probabilistic precedence satisfaction under those models. These safety margins are integrated into \planner to form \uncertainPlanner, which maintains collision-free execution under the adopted uncertainty regimes.

\subsection{MAEP under Uncertainty}\label{sec:maep_under_uncertainty}
We extend the MAEP in~\cref{sec:maep} by incorporating an execution timing uncertainty model for movement.
Concretely, we model such uncertainty as additive deviations to nominal traversal times. Specifically, if $T_i^k$ is the nominal time for agent $a_i$ to move from $q_i^k$ to $q_i^{k+1}$, then the realized time is
\begin{equation}
    \hat{T}_i^{k} = T_i^{k} + \delta_i^k, \label{eq:noise}
\end{equation}
where $\delta_i^k$ is a temporal perturbation.

In this work, we adopt three representative execution timing uncertainty models from the MAPF, stochastic control, and temporal reasoning literature:

\begin{enumerate}
    \item \textbf{Bounded Cumulative Delay Model.}  
    This model represents bursty but bounded execution delays, such as those caused by transient control saturation or local disturbances. Bounded cumulative delay models have been used in MAPF under uncertainty because they impose a finite delay budget per agent without requiring full stochastic assumptions~\cite{atzmon2020robust}. Formally,
    \[
        \delta_i^k =
        \begin{cases}
            \eta & \text{with prob. } p,\ \text{if } \sum_{j=0}^{k-1}\delta_i^j < \mathcal{K}_i,\\
            0 & \text{otherwise},
        \end{cases}
    \]
    where $p$ is the probability of delay per-move, $\eta > 0$ is the magnitude of the unit delay, and $\mathcal{K}_i$ is the total accumulated delay limit for the agent. Intuitively, each time an agent moves, it may experience a fixed delay $\eta$ with probability $p$. However, the total accumulated delay across all moves is capped at $\mathcal{K}_i$. We assume $\mathcal{K}_i$ is a multiple of $\eta$, i.e., $\mathcal{K}_i = \ell_i\cdot\eta$ for some $\ell_i \in \mathbb{N}$.

    \item \textbf{Bounded Uniform Noise Model.}  
    In temporal reasoning under bounded intervals (e.g., Simple Temporal Networks), traversal times are assumed to lie within a known interval~\cite{dechter1991temporal}. We model this as
    \[
        \delta_i^k \sim \mathcal{U}(-\epsilon_i,\epsilon_i),
    \]
    where $\mathcal{U}(-\epsilon_i,\epsilon_i)$ denotes the uniform distribution over $[-\epsilon_i, \epsilon_i]$, and $\epsilon_i>0$ defines the maximum deviation from nominal traversal time.

    \item \textbf{Gaussian Noise Model.}  
    We model aggregate timing variability arising from many independent small perturbations as additive zero-mean Gaussian noise. Gaussian noise models are widely used in stochastic control and robotics to represent unbounded execution uncertainty~\cite{van2011lqg}:
    \[
        \delta_i^k \sim \mathcal{N}(0,\sigma^2),
    \]
    where $\mathcal{N}(0,\sigma^2)$ denotes the normal distribution with zero mean and variance $\sigma^2$, and the variance scales with the traversal distance.
\end{enumerate}

Unless otherwise stated, perturbations $\delta_i^k$ are assumed independent across agents and time. Correlated and heavy-tailed noise are outside the scope of the current derivation.

\subsection{Probabilistic Safety Margin Formulation}

Let $\mu_i^k = \sum_{j=0}^{k-1} T_i^j$ denote the nominal reach time of $a_i$ at vertex $v_i^k$, and let $\Delta_i^k = \sum_{j=0}^{k-1} \delta_i^j$ be the cumulative perturbation. Then the actual reach time is
\begin{equation}
    t_i^r(v_i^k) = \mu_i^k + \Delta_i^k.
\end{equation}

For a Type-2 edge $(v_i^{m+1}, v_j^n)$ between agents $a_i$ and $a_j$ with $loc(v_i^m) = loc(v_j^n)$, we require that the precedence constraint is satisfied with probability at least $P_d$, the \emph{safety probability threshold}:
\begin{equation}
    P\bigl(t_i^l(v_i^m) < t_j^r(v_j^n)\bigr) \;\geq\; P_d. \label{eq:safety_req}
\end{equation}

Under the adopted uncertainty model and independence assumptions, this probabilistic constraint can be converted into a safety margin $\gamma_{i,j}^{m,n}(P_d)$ such that
\begin{equation}
    \mu_j^n - \mu_i^{m+1} \;\geq\; \gamma_{i,j}^{m,n}(P_d). \label{eq:safety_margin}
\end{equation}

The safety margin depends on the distribution (or bounds) of the difference $\Delta_i^{m+1} - \Delta_j^n$ under the selected uncertainty model.
Since the first two models have bounded perturbations, we can derive safety margins that guarantee deterministic precedence satisfaction, i.e., $P_d = 1$. The Gaussian model has unbounded perturbations and thus only admits probabilistic guarantees for a given $P_d < 1$.

\begin{enumerate}
    \item \textbf{Bounded Cumulative Delay Model.}  
    Under the bounded cumulative delay model, each agent’s delay at each step is either $0$ or $\eta$, and the total delay budget is $\mathcal{K}_i$ for agent $a_i$ and $\mathcal{K}_j$ for agent $a_j$.
    Therefore, the cumulative perturbations satisfy
\begin{equation}
    \Delta_i^{m+1}-\Delta_j^n\;\in\;[-\min(n,\ell_j)\,\eta,\ \min(m{+}1,\ell_i)\,\eta].
\end{equation}
 Thus, a valid deterministic safety margin ($P_d=1$) is:
    \begin{equation}\label{eq:bounded_delay}
    \gamma_{i,j}^{m,n}(1)=\min(m{+}1,\ell_i)\,\eta,
    \end{equation}
    which accounts for the maximum possible prefix delay agent $a_i$ could incur under its finite budget.

    \item \textbf{Bounded Uniform Noise Model.}
    When $\delta_i^k\sim\mathcal{U}(-\epsilon_i,\epsilon_i)$ and $\delta_j^k\sim\mathcal{U}(-\epsilon_j,\epsilon_j)$, each term $\delta_i^k$ and $\delta_j^k$ lies within $[-\epsilon,\epsilon]$. Consequently,
    \begin{equation}
        \Delta_i^{m+1}-\Delta_j^n \;\in\;[-(m{+}1)\epsilon_i-n\epsilon_j,\,(m{+}1)\epsilon_i+n\epsilon_j].
    \end{equation}
    For deterministic precedence satisfaction ($P_d=1$), a valid safety margin is
    \begin{equation}\label{eq:stn_noise}
        \gamma_{i,j}^{m,n}(1) = (m{+}1)\epsilon_i+n\epsilon_j,
    \end{equation}
    which corresponds to the worst-case bound. 

    \item \textbf{Gaussian Noise Model.}  
    Under additive zero-mean Gaussian noise, if $\delta_i^k\sim\mathcal{N}(0,\sigma_i^2)$ and $\delta_j^k\sim\mathcal{N}(0,\sigma_j^2)$ independently, then $\Delta_i^{m+1}\sim\mathcal{N}(0,(m{+}1)\sigma_i^2)$ and $\Delta_j^n\sim\mathcal{N}(0,n\sigma_j^2)$. The difference $\Delta_i^{m+1}-\Delta_j^n$ is therefore Gaussian with variance
    \begin{equation}
        \mathrm{Var}[\Delta_i^{m+1}-\Delta_j^n] = (m{+}1)\sigma_i^2 + n\sigma_j^2.
    \end{equation}
    The $(1-P_d)$-quantile of this distribution gives
    \begin{equation}\label{eq:gaussian_noise}
        \gamma_{i,j}^{m,n}(P_d) = \Phi^{-1}(P_d)\cdot\sqrt{(m{+}1)\sigma_i^2+n\sigma_j^2},
    \end{equation}
    where $\Phi^{-1}$ is the inverse cumulative distribution function of the standard normal.
\end{enumerate}

\subsection{Integration with \planner}

To integrate uncertainty into the planning framework, \uncertainPlanner enforces a minimum time gap between agents’ expected leave and reach times at shared locations based on the computed safety margins. Specifically, once $t_i^l(v_i^m)=t_i^r(v_i^{m+1})$ is fixed for agent $a_i$, the earliest permissible reach time for agent $a_j$ at $v_j^n$ is
\begin{equation}
\label{eq:integrate_ktpg}
    t_j^r(v_j^n)\;\geq\;t_i^r(v_i^{m+1})+\gamma_{i,j}^{m,n}(P_d).
\end{equation}

Reserved intervals are then updated by intersecting these constrained windows with the existing feasibility intervals for both agents, thereby maintaining the precedence under temporal uncertainty. This mechanism preserves collision-free execution under the adopted uncertainty models.\\

\noindent\textbf{Remark.}
The choice of execution timing uncertainty model depends on the robot platform, controller, and operating environment. Our framework is mainly intended to handle execution delays arising from routine variability, such as tracking errors or local disturbances, which are typically bounded or light-tailed. More severe heavy-tailed disturbances, e.g., robot breakdowns or major system faults, are not explicitly modeled here and should instead be handled by higher-level path replanning or task reassignment.

\begin{figure}
    \centering
    \includegraphics[width=1.\linewidth]{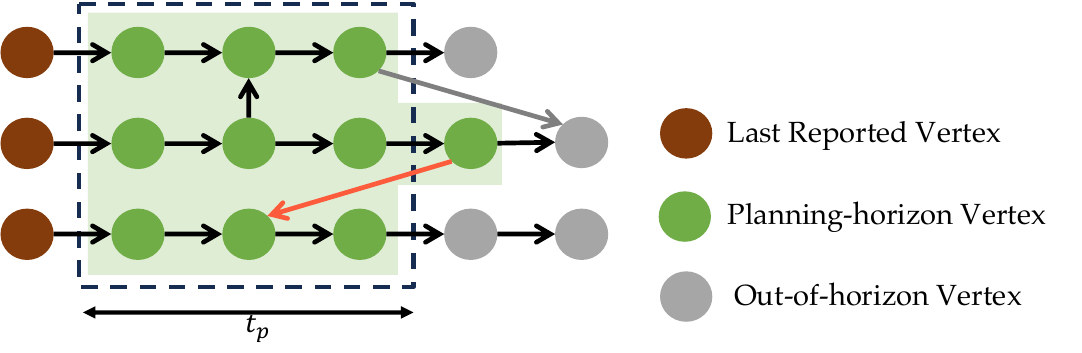}
    \caption{Illustration of the windowed robust execution.}
    \label{fig:windowed_execution}
\end{figure}

\section{Windowed kTPG (\mymethod)} \label{sec:winktpg}
While adding safety margins ensures the speed profiles are collision-free with a certain probability, as shown in~\cref{eq:bounded_delay,eq:stn_noise,eq:gaussian_noise}, the uncertainty in the reach times of an agent increases along its path.
So, according to~\cref{eq:integrate_ktpg}, the required safety margin for Type-2 edges between later vertices may become large, compromising solution quality.
To address this, we propose using online execution information to reduce estimated execution timing uncertainty.

\subsection{Real-Time MAEP under Uncertainty}
We extend the MAEP with uncertainty in \cref{sec:maep_under_uncertainty} by including information updates during execution, as proposed in~\cite{honig2019warehouse}. We assume that agents report their actual reach times and current kinodynamic states upon arriving at each location.

\subsection{\mymethod}
To maintain robustness while improving solution quality, we propose \mymethod.
\mymethod replaces the estimated reach times with the actual reach times reported by agents.
Using these actual times, it recalculates the estimated reach times for subsequent vertices, reducing their uncertainty.
\mymethod employs a windowed replanning method~\cite{berndt2023receding} to incorporate this updated information.
Concretely, we trigger replanning every $t_e$ seconds, during which we invoke \uncertainPlanner to generate speed profiles within a defined planning window.
A planning window, shown as the green region in~\cref{fig:windowed_execution}, is a subset of the original TPG and spans at least $t_p$ vertices for each agent.

When replanning is triggered, we follow the method from~\cite{berndt2023receding} to determine which vertices to include in the planning window. For each agent, we identify its most recently reported vertex and the corresponding reach time. The next $N_E$ vertices, with $N_E$ as a predefined hyper-parameter, are then marked as \emph{enqueued}, meaning the agent can continue to follow its previous speed profile through these vertices without replanning.
The planning window begins immediately after the last enqueued vertex. As shown in~\cref{fig:windowed_execution}, we first include the next $t_p$ vertices for each agent in the window.
To ensure all precedence constraints are satisfied, we further check Type-2 edges within the window.
If any such edge originates from outside the current window (e.g., the orange edge in~\cref{fig:windowed_execution}), we recursively expand the window to include the source vertex and any earlier vertices needed to preserve execution order.

Once the planning window is finalized, we use \uncertainPlanner to compute speed profiles for all agents within it.
Using the last reported vertex $v_{i}^{p}$ and its corresponding reach time $t_{i}^r(v_i^p)$ for each agent, we calculate the updated reach time for subsequent vertices as:
$t_i^r(v_i^k) = t_i^r(v_i^p) + \sum_{j=p}^{k-1}{\hat{T}_{j}}$.
This reduces the temporal uncertainty in the estimated reach time.
Once the speed profiles are generated, they are sent to the agents for execution.
This replanning and execution process repeats until all agents reach their goals.

\section{Theoretical Analysis} \label{sec:theory}
\begin{theorem} [Completeness of \planner and \uncertainPlanner]
\label{theorem:1}
Given a collision-free and cycle-conflict-free MAPF plan, 
\planner and \uncertainPlanner terminate in finite time and produce 
a kinodynamically feasible solution that preserves all precedence constraints.
\end{theorem}
\begin{proofsketch}
We first establish the existence of a feasible solution by leveraging proof from~\cite{honig2016multi,honig2019warehouse}.
As shown in these works, if the original MAPF plan is collision-free and cycle-conflict-free, then a kinodynamically feasible set of speed profile that preserves all precedence constraints exists.
We now prove termination of \planner.
As shown in~\cref{method:conflict_edge}, our method of determining the split time ensures that we can always find a valid speed profile in each iteration. Since $\mathcal{G}$ is guaranteed to be a directed acyclic graph, if there are locked vertices in $\mathcal{G}$, there is at least one conflicting edge originating from an unlocked vertex to a locked vertex.
So when we replan for the agent associated with this unlocked vertex, \planner guarantees progress by satisfying at least this conflicting edge. Therefore, \planner is guaranteed to unlock all vertices after a finite number of iterations. That is, kTPG is complete and guarantees to find feasible solutions.
For \uncertainPlanner, incorporating a safety margin does not affect the existence of a valid speed profile in each iteration, making the completeness of \planner unchanged.
\end{proofsketch}


\begin{theorem} [Completeness of \mymethod]
Given a collision-free and cycle-conflict-free MAPF plan, \mymethod ensures that, under the adopted execution timing uncertainty model, all agents reach their goal locations in finite time, while precedence satisfaction and collision-free execution are preserved according to the guarantees provided by that model.
\end{theorem}
\begin{proofsketch}
In each planning window, feasibility is ensured by allowing each agent to retain and follow its previously computed speed profile up to the last vertex from the previous window.
Based on the reasoning in~\cref{method:conflict_edge}, we ensure that, for each agent, a valid speed profile exists in the current window, enabling it to wait at the final vertex of the previous window before proceeding.
As shown in~\cref{theorem:1}, at least one agent makes progress in each set of speed profiles. Thus, \mymethod completes in finite time with a bounded number of \uncertainPlanner calls.
\end{proofsketch}

\begin{theorem} [Time Complexity of \planner and \uncertainPlanner]
    \planner and \uncertainPlanner require at most $\min{(\left|\mathcal{V}\right|, \left|\mathcal{E}_2\right|)}$ iterations of speed profile updates.
\end{theorem}
\begin{proof}
As proved in~\Cref{theorem:1}, each iteration satisfies at least one new Type-2 edge and unlocks at least one previously locked vertex.
This implies that the total number of iterations is upper bounded by $\min(|\mathcal{V}|, |\mathcal{E}_2|)$.
\end{proof}
Since \mymethod calls \uncertainPlanner within each of the $W$ planning windows, the total number of iterations is at most $W \cdot \min(|\mathcal{V}|, |\mathcal{E}_2|)$.


\section{Empirical Evaluation} \label{sec:exps}

\begin{figure*} [ht!]
    \centering
    \includegraphics[width=1.\linewidth]{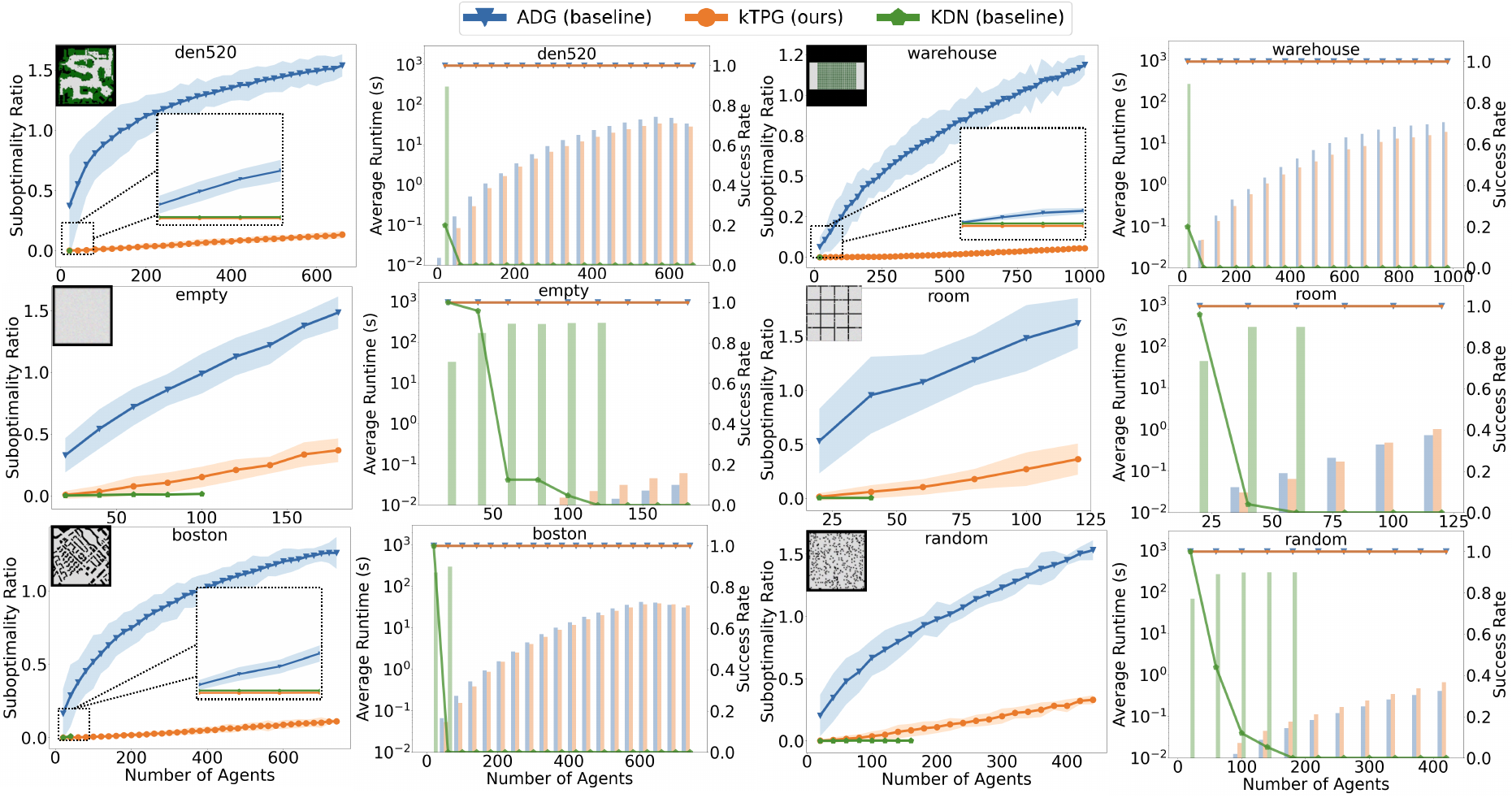}
    \caption{Average suboptimality (odd-numbered columns), average runtime (bar plots in even-numbered columns), and success rate (line plots in even-numbered columns).}
    \label{fig:exp1_sol_sr}
\end{figure*}

In this section, we empirically evaluate the proposed methods. First, we compare \planner with prior MAPF execution methods in terms of solution quality and scalability under kinodynamic constraints. Second, we study \mymethod under different execution timing uncertainty models. Third, we validate \mymethod in more realistic settings using both physics-based simulation and physical-robot experiments.

We evaluate all methods on six maps with 25 instances from the MovingAI MAPF benchmark~\cite{Stern2019benchmark}: 
\texttt{empty} (\textit{empty-32-32}), \texttt{random} (\textit{random-64-64-10}), \texttt{room} (\textit{room-64-64-16}), \texttt{den520d} (\textit{den520d}), and \texttt{Boston} (\textit{Boston\_0\_256}), \texttt{warehouse} (\textit{warehouse-20-40-10-2-2}).
The distance between two neighboring cells is $1\; m$. 
The MAPF plans used for execution are generated by PBS~\cite{ma2019searching}.
We use \emph{suboptimality ratio} to evaluate the solution quality, which is defined as $\frac{T_{sum}-T_{ideal}}{T_{ideal}}$, where $T_{sum}$ is the sum of reach times, and $T_{ideal}$ is the sum of individual agent reach times without considering collisions or delays.

Unless otherwise specified, all methods included in the study were implemented in C++ and tested on an Ubuntu 20.04 system with an AMD 3990x processor and 188 GB of memory. The code was run using a single core for all computations. The source code for our method is publicly available at \repolink.

\subsection{Kinodynamic Planning}
In this experiment, we evaluate the solution quality of the speed profiles generated by \planner and compare it against two baselines: ADG and KDN.
ADG~\cite{honig2019warehouse} generates speed profiles by planning only for vertices without unsatisfied incoming Type-2 edges, using the same speed profile planner as \planner.
Furthermore, ADG is assumed to have access to real-time agent locations and replans its speed profile every $0.01\, s$.
KDN~\cite{zhang2021temporal} formulates the temporal dependency of TPG as a MILP problem and finds the optimal solution for MAEP~(see \cref{sec:maep}).

\paragraph{Experiment Setup}
Since KDN is designed for an omnidirectional robot model without uncertainty, we adopt the same setting in this experiment.
An \textit{omnidirectional robot} can move in any orientation, with its speed profile constrained by the dynamic limits of its physical model.
In this experiment, each agent is subject to a speed limit of $[0, 2]\, \text{m/s}$ and an acceleration limit of $[-1, 1] \, \text{m/s}^2$.
As required by the speed profile planner, speeds at the vertices are discretized into three levels: $\{0, \sqrt{2}, 2\}\, \text{m/s}$.
The runtime limit for all methods is 300 seconds.

\paragraph{Result Analysis}
As shown in~\cref{fig:exp1_sol_sr}, both \planner and KDN show significant improvements in solution quality across all maps compared to ADG. This improvement is more evident on larger maps with more agents (e.g., up to 51.7\% on the warehouse map). This is because ADG only plans speed profiles for vertices without incoming Type-2 edges. In situations with more Type-2 edges, this strategy is overly conservative and leads to frequent unnecessary deceleration.
While KDN achieves slightly better solution quality, it is much less scalable than \planner. Both \planner and ADG maintain a 100\% success rate on problems involving up to 1,000 agents.
In contrast, KDN, while ensuring optimality, cannot solve problems with more than 200 agents. This is due to KDN relying on a MILP formulation to encode the precedence constraints imposed by Type-2 edges, whose complexity grows exponentially with the number of Type-2 edges.
Notably, while the runtime for each planning round of ADG is low, its frequent replanning significantly increases the total runtime, making it similar to that of \planner.

\begin{table*}[t]
\caption{Additional reliability and efficiency metrics on the warehouse map with 100 agents under three execution timing uncertainty models ($t_e = 10$, $t_p=20$). Results are averaged over all scenarios.}
\centering
\begin{tabular}{lcccccccc}
\toprule
\multirow{2}{*}{Noise level} & \multicolumn{2}{c}{Violation rate (\%)} 
& \multicolumn{2}{c}{Avg. velocity (m/s)} 
& \multicolumn{2}{c}{Total idle time (s)} 
& \multicolumn{2}{c}{Total stop-and-go count} \\
\cmidrule(lr){2-3} \cmidrule(lr){4-5} \cmidrule(lr){6-7} \cmidrule(lr){8-9}
 & ADG & \mymethod & ADG & \mymethod & ADG & \mymethod & ADG & \mymethod \\
\midrule
$\sigma=0.001$ & 0.00 & 0.74 & 0.95 & 1.02 & 3460.00 & 3104.22 & 3920.88 & 3317.72 \\
$\sigma=0.03$  & 0.00 & 0.70 & 0.95 & 1.01 & 3481.46 & 3236.12 & 3920.52 & 3320.28 \\
$\sigma=0.05$  & 0.00 & 0.69 & 0.95 & 1.01 & 3497.71 & 3330.00 & 3919.40 & 3322.28 \\
\midrule
$\mathcal{K}=1.0$ & 0.00 & 0.00 & 0.95 & 1.02 & 3453.82 & 3164.25 & 3920.48 & 3324.00 \\
$\mathcal{K}=3.0$ & 0.00 & 0.00 & 0.94 & 1.01 & 3504.14 & 3227.75 & 3922.76 & 3327.24 \\
$\mathcal{K}=5.0$ & 0.00 & 0.00 & 0.93 & 0.99 & 3544.29 & 3305.45 & 3923.52 & 3328.58 \\
\midrule
$\epsilon=0.01$ & 0.00 & 0.00 & 0.95 & 1.02 & 3461.88 & 3147.40 & 3920.08 & 3320.24 \\
$\epsilon=0.02$ & 0.00 & 0.00 & 0.95 & 1.02 & 3468.71 & 3196.95 & 3922.12 & 3323.28 \\
$\epsilon=0.1$  & 0.00 & 0.00 & 0.95 & 0.99 & 3506.42 & 3616.23 & 3920.04 & 3345.20 \\
\bottomrule
\end{tabular}
\label{table:additional_metrics}
\end{table*}

\subsection{Robust Execution Under Uncertainty}
In this section, we evaluate the performance of \mymethod in environments with execution timing uncertainty. As KDN cannot handle uncertainty, only ADG is used as a baseline.
Our goal is to examine whether the proposed safety margin mechanism can provide the same guarantees predicted by the theory. We also study its execution efficiency under different factors, including window size and uncertainty level.

\begin{figure}[t!]
    \centering
    \includegraphics[width=1.\linewidth]{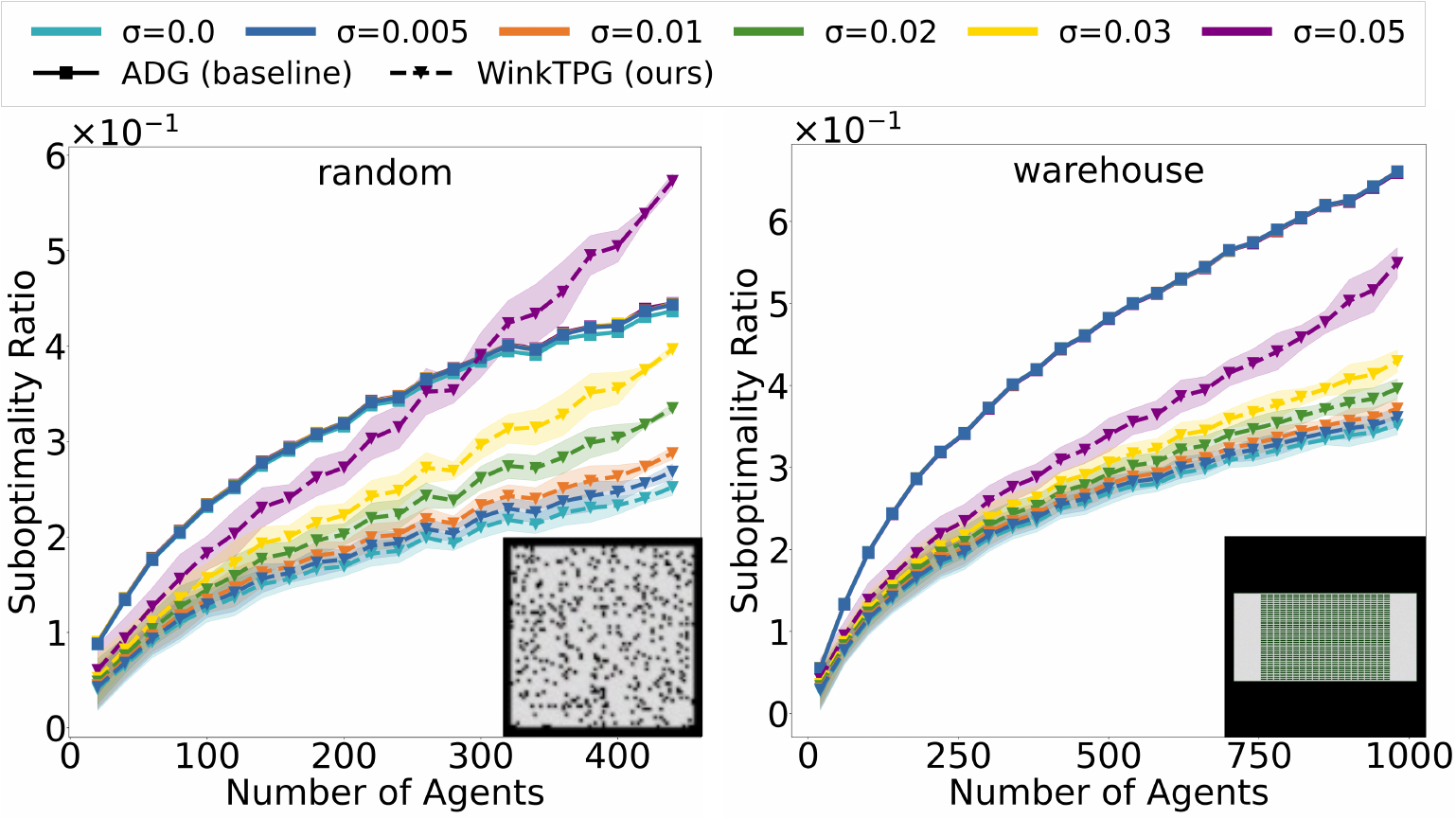}
    \caption{Comparison of solution quality between ADG (solid lines) and \mymethod (dashed lines) under the Gaussian noise model across different uncertainty levels (indicated by color). The ADG lines largely overlap across uncertainty levels in both figures. Shaded regions indicate the variance across all scenarios.}
    \label{fig:exp2_noise}
\end{figure}

\begin{figure}[t!]
    \centering
    \includegraphics[width=1.\linewidth]{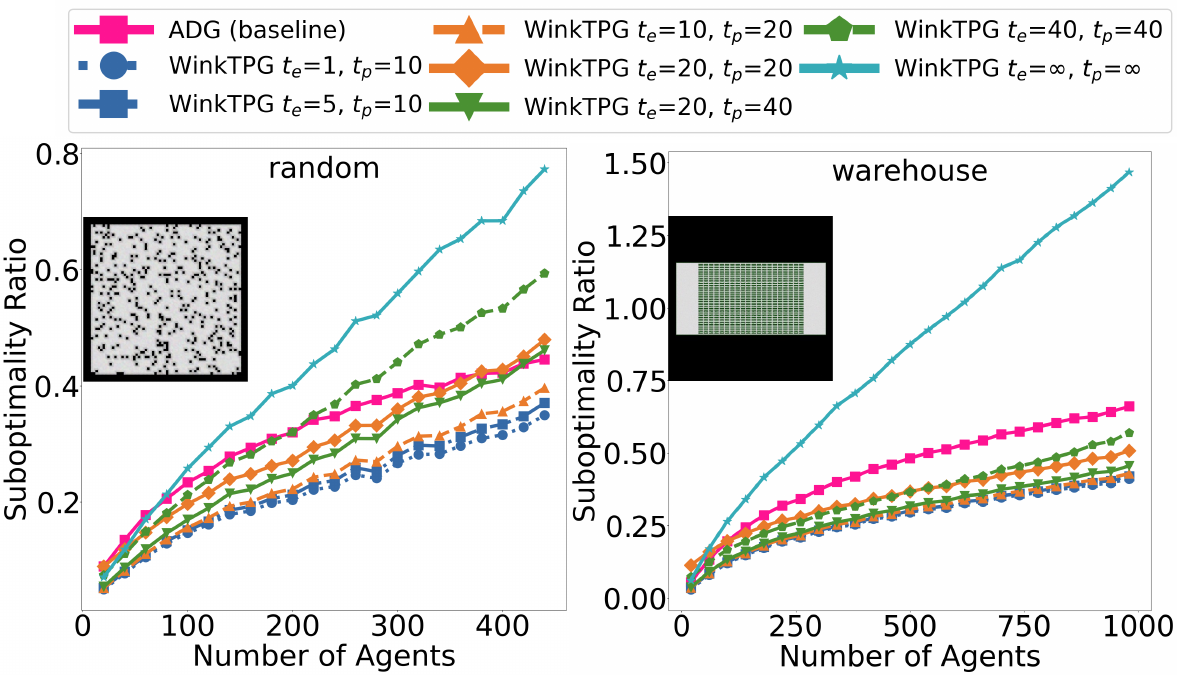}
    \caption{Comparison of solution quality between ADG and \mymethod under the Gaussian noise model, with different execution and planning windows.}
    \label{fig:exp2_horizon}
\end{figure}

\paragraph{Experiment Setup}
We use a differential-drive robot model that is more commonly used in real-world applications.
Compared to omnidirectional robots, it moves only along its current orientation, while being constrained by the same speed and acceleration limits as in the previous section.
It can change its orientation only when its velocity is zero.
Specifically, it takes $0.5 \, \text{s}$ to turn left or right and $0.9 \, \text{s}$ to turn backward.
To evaluate both reliability and execution efficiency, we report the following additional metrics in~\cref{table:additional_metrics}: (1) \emph{Precedence violation rate:} the ratio of violated precedence constraints at shared locations to the total number of such constraints.
(2) \emph{Average velocity:} the mean linear velocity over the entire execution horizon.
(3) \emph{Total idle time:} the cumulative duration during which the linear and angular velocity of an agent is zero.
(4) \emph{Stop-and-go count:} the number of times an agent accelerates from zero velocity to positive velocity.

\subsubsection{Unbounded Uncertainty Model}

We first evaluate \mymethod under a Gaussian noise model.  
We parameterize the noise magnitude using a scalar $\sigma$ shared across all agents
By~\cref{eq:noise}, this implies that the traversal time between adjacent vertices (distance $1\,\text{m}$) has variance $\sigma^2$.
For intuition, the nominal traversal time between adjacent vertices at maximum speed is $\frac{1 \, \text{m}}{2 \, \text{m/s}} = 0.5 \, \text{s}$. 
When $\sigma=0.05$, applying the $3\sigma$ rule~\cite{casella2024statistical} implies up to $\pm 30\% \times 0.5\,\text{s}$ deviation. 
Since this already represents substantial timing variability for short-range motion, we restrict $\sigma \le 0.05$.
We set the safety probability threshold to $P_d = 0.95$ unless otherwise specified.

\noindent
\emph{Effect of Uncertainty Magnitude.}
We first evaluate \mymethod under varying uncertainty levels using a fixed window size. Specifically, the execution window \(t_e\) is set to \(10 \, \text{s}\), and the planning window \(t_p\) is \(20\).
As shown in~\cref{fig:exp2_noise}, while \mymethod generates more suboptimal solutions with larger $\sigma$, it consistently outperforms ADG in environments with moderate temporal uncertainty (i.e., $\sigma < 0.03$).
This trend is further supported by the additional metrics in~\cref{table:additional_metrics}. Compared with ADG, \mymethod generally achieves higher average velocity, lower total idle time, and fewer stop-and-go events under the Gaussian noise model, indicating smoother and more efficient motion.
This supports our earlier claim that ADG is an overly conservative strategy.
The solution quality of ADG remains nearly constant across different uncertainty levels because it replans at a very high frequency and is thus less sensitive to temporal uncertainty.
As shown later, \mymethod could achieve better solution quality with a smaller execution window.

\noindent
\emph{Effect of Window Size.}
Next, we evaluate the relationship between solution quality and the length of the execution and planning windows. All experiments are conducted in environments with $\sigma = 0.03$.
As shown in~\cref{fig:exp2_horizon}, configurations where the execution window is smaller than the planning window consistently show better solution quality. Also, smaller execution windows always achieve better performance, due to more frequent replanning with updated location information.

\begin{figure}[t]
    \centering
    \includegraphics[width=1.0\linewidth]{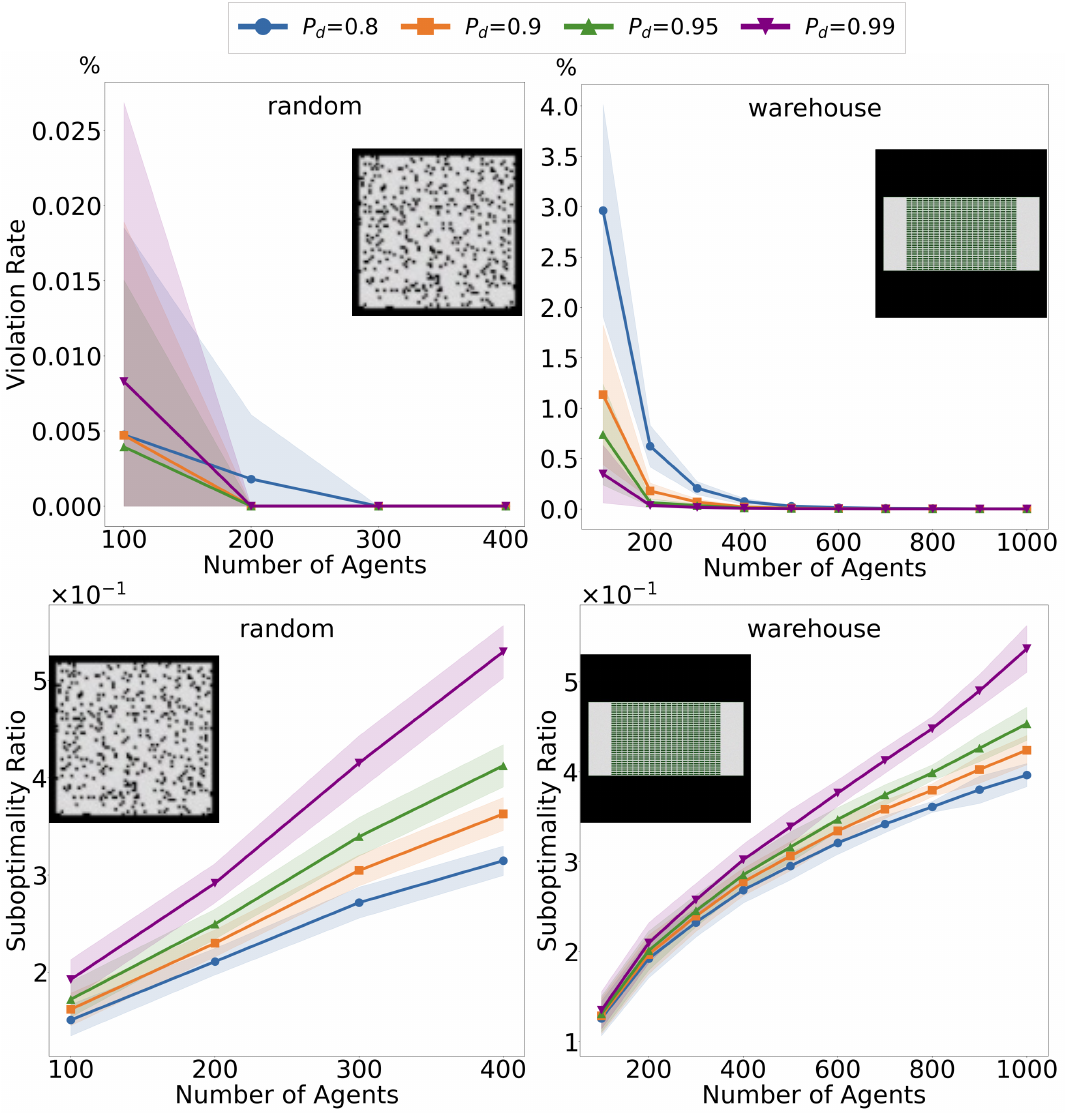}
    \caption{Empirical precedence violation rate and solution quality as a function of $P_d$ under the Gaussian noise model.}
    \label{fig:pd_vs_violation}
\end{figure}

\noindent
\emph{Effect of Safety Probability Threshold.}
We now examine how the theoretical safety probability threshold $P_d$ translates into empirical precedence violation rates under the Gaussian noise model.
For each map and uncertainty level, we vary $P_d \in \{0.8, 0.9, 0.95, 0.99\}$.
\Cref{fig:pd_vs_violation} shows that the empirical violation rate remains below the theoretical upper bound $1 - P_d$, confirming that the probabilistic safety margin behaves consistently with the Gaussian model assumptions, while the slight non-monotonicity on the random map is likely attributable to variance across trials.
Meanwhile, larger $P_d$ values produce more conservative safety margins and therefore increased execution time, showing a clear trade-off between robustness and efficiency.

\begin{table*}[t]
\caption{Planning Runtime on the warehouse map with $\sigma=0.03$.
MAPF refers to the time taken to generate the initial MAPF plan, while ADG and \mymethod represent the runtime per planning round.
For \mymethod, we also report the actual average planning-window size $\bar{t_p}$, which denotes the average number of vertices included in the planning window during replanning.}
\centering
\begin{tabular}{ccccc}
\toprule
{Number of Agents} & \multicolumn{1}{c}{{MAPF}} & \multicolumn{1}{c}{{ADG}} & \multicolumn{2}{c}{{\mymethod}} \\
\midrule
\multirow{3}{*}{100}
& \multirow{3}{*}{$89.08 \pm 32.76\, ms$}
& \multirow{3}{*}{\makecell{$0.24 \pm 0.03\, ms$}}
& {$t_p$=10} & $2.17 \pm 0.20\, ms$ ($\bar{t_p}=11.30 \pm 0.09$) \\
& & & {$t_p$=20} & $3.31 \pm 0.43\, ms$ ($\bar{t_p}=21.31 \pm 0.10$) \\
& & & {$t_p$=40} & $4.96 \pm 0.78\, ms$ ($\bar{t_p}=41.34 \pm 0.11$)\\
 & & & {$t_p$=$\infty$} & $266.06 \pm 52.33 \, ms$ ($\bar{t_p}=N/A$)\\
\midrule
\multirow{3}{*}{1,000}
& \multirow{3}{*}{$44513.62 \pm 3941.23\, ms$}
& \multirow{3}{*}{\makecell{$2.02 \pm 0.09\, ms$}}
& {$t_p$=10} & $343.85 \pm 14.32\, ms$ ($\bar{t_p}=12.62 \pm 0.08$)\\
& & & {$t_p$=20} & $389.23 \pm 14.21\, ms$ ($\bar{t_p}=22.68 \pm 0.08$)\\
& & & {$t_p$=40} & $519.42 \pm 20.13\, ms$ ($\bar{t_p}=42.82 \pm 0.09$)\\
 & & & {$t_p$=$\infty$}  & $28211.74 \pm 1673.22\, ms$ ($\bar{t_p}=N/A$)\\
\bottomrule
\end{tabular}
\label{table:runtime}
\end{table*}

\begin{table*}[t]
\caption{Communication overhead on the warehouse map ($t_{e}=10$, $t_{p}=20$). Uplink denotes the total state updates sent from agents to the server,
downlink denotes the total action commands sent from the server to agents,
and replan freq.\ denotes the number of replanning cycles triggered per second.}
\centering
\begin{tabular}{ccrrc}
\toprule
{Number of Agents} & \multicolumn{1}{c}{{Method}} & \multicolumn{1}{c}{{Uplink (KB/replan)}} & \multicolumn{1}{c}{{Downlink (KB/replan)}} & \multicolumn{1}{c}{{Replan Freq.\ (Hz)}} \\
\midrule
\multirow{2}{*}{100}
& ADG       & $1.333 \pm 0.111$  & $2.134 \pm 0.183$    & $1.0$\\
& \mymethod & $1.419 \pm 0.099$  & $19.953 \pm 1.384$   & $0.1$\\
\midrule
\multirow{2}{*}{1,000}
& ADG       & $13.474 \pm 0.361$ & $15.688 \pm 0.461$   & $1.0$\\
& \mymethod & $13.258 \pm 0.339$ & $178.971 \pm 4.901$  & $0.1$\\
\bottomrule
\end{tabular}
\label{tab:communication-overhead}
\end{table*}


\noindent
\emph{Planning Runtime and Communication Overhead.}
We report in~\cref{table:runtime} the runtime of ADG, \mymethod, and the MAPF planner. 
In both cases, ADG and \mymethod demonstrate significantly faster runtimes compared to the MAPF planner used to generate the paths.
Notably, while ADG appears faster, it requires significantly more frequent replanning.
Additionally, with an appropriately chosen planning window size, \mymethod generates plans in under a second, showing it is capable of near real-time performance.
We further quantify the communication overhead in~\cref{tab:communication-overhead}.
Although \mymethod transmits larger messages per replanning cycle due to coordinated speed profiles, it replans less frequently than ADG.
At 1{,}000 agents, the downlink per plan is approximately 179 \, KB, which remains well within the capacity of standard wireless networks.

\subsubsection{Bounded Uncertainty Models}

We next evaluate \mymethod under the two bounded uncertainty models introduced in Sec.~\ref{sec:ktpgu}.

Unlike the Gaussian noise model, both bounded models admit deterministic safety guarantees, since the cumulative perturbations are bounded by construction.

\begin{figure}[t!]
    \centering
    \includegraphics[width=1.\linewidth]{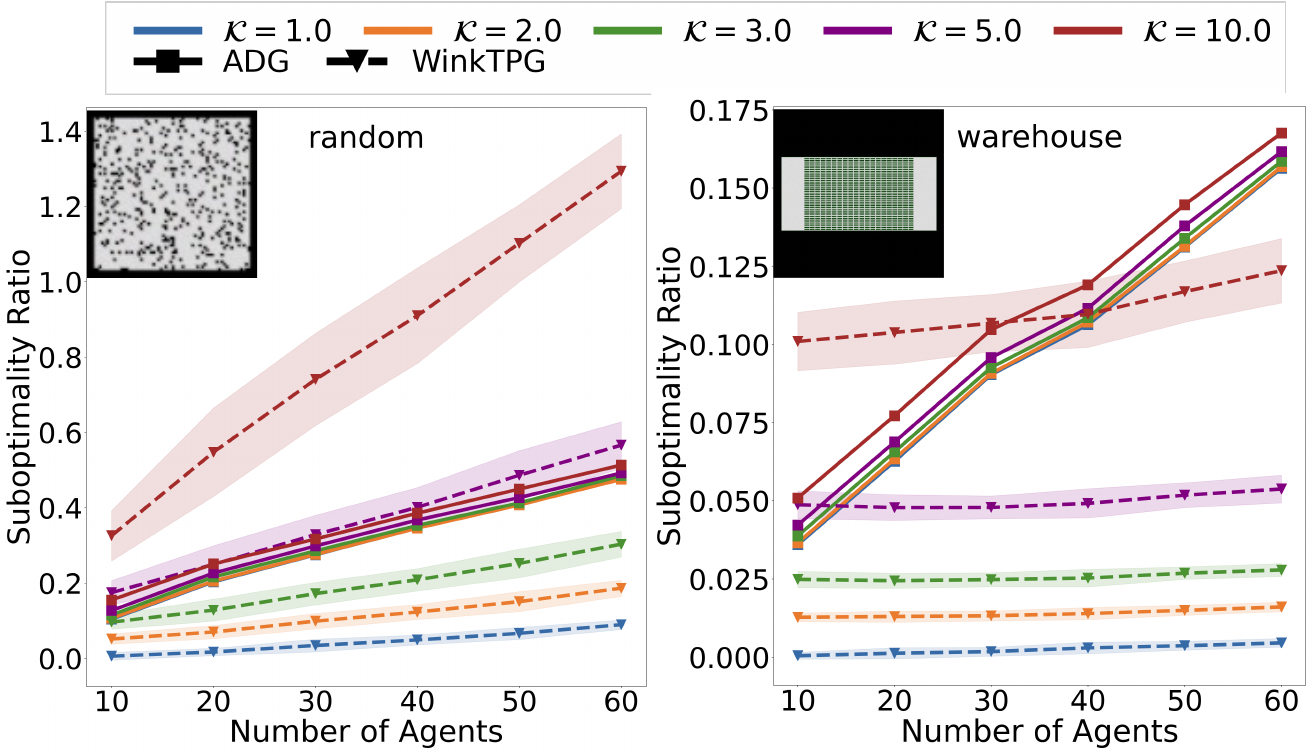}
    \caption{Comparison of solution quality between ADG (solid lines) and \mymethod (dashed lines) under the bounded cumulative delay model across different cumulative delay budget $\mathcal{K}$ (indicated by color).}
    \label{fig:kdelay_results}
\end{figure}

\noindent
\emph{Results under Bounded Cumulative Delay Model.}
For the bounded cumulative delay model, we vary the cumulative delay budget
$\mathcal{K} \in \{1.0, 2.0, 3.0, 5.0, 10.0\}\, s$ and fix the unit delay $\eta = 1.0\, s$. 
This parameter represents the maximum total execution delay (in seconds) an agent may accumulate along its path.
As shown in ~\cref{fig:kdelay_results,table:additional_metrics}, \mymethod achieves 100\% precedence satisfaction across all tested $\mathcal{K}$ values, confirming the deterministic guarantee derived in ~\cref{sec:ktpgu}.
In terms of solution quality, \mymethod consistently achieves lower suboptimality than ADG with smaller $\mathcal{K}$ values.
The performance gap becomes more pronounced for moderate delay budgets (e.g., $\mathcal{K}=1.0$ and $\mathcal{K}=2.0$), where ADG’s conservative, event-driven execution causes unnecessary deceleration.
As $\mathcal{K}$ increases, the solution cost increases gradually due to larger safety margins required to accommodate worst-case cumulative delays.

\begin{figure}[t!]
    \centering
    \includegraphics[width=1.\linewidth]{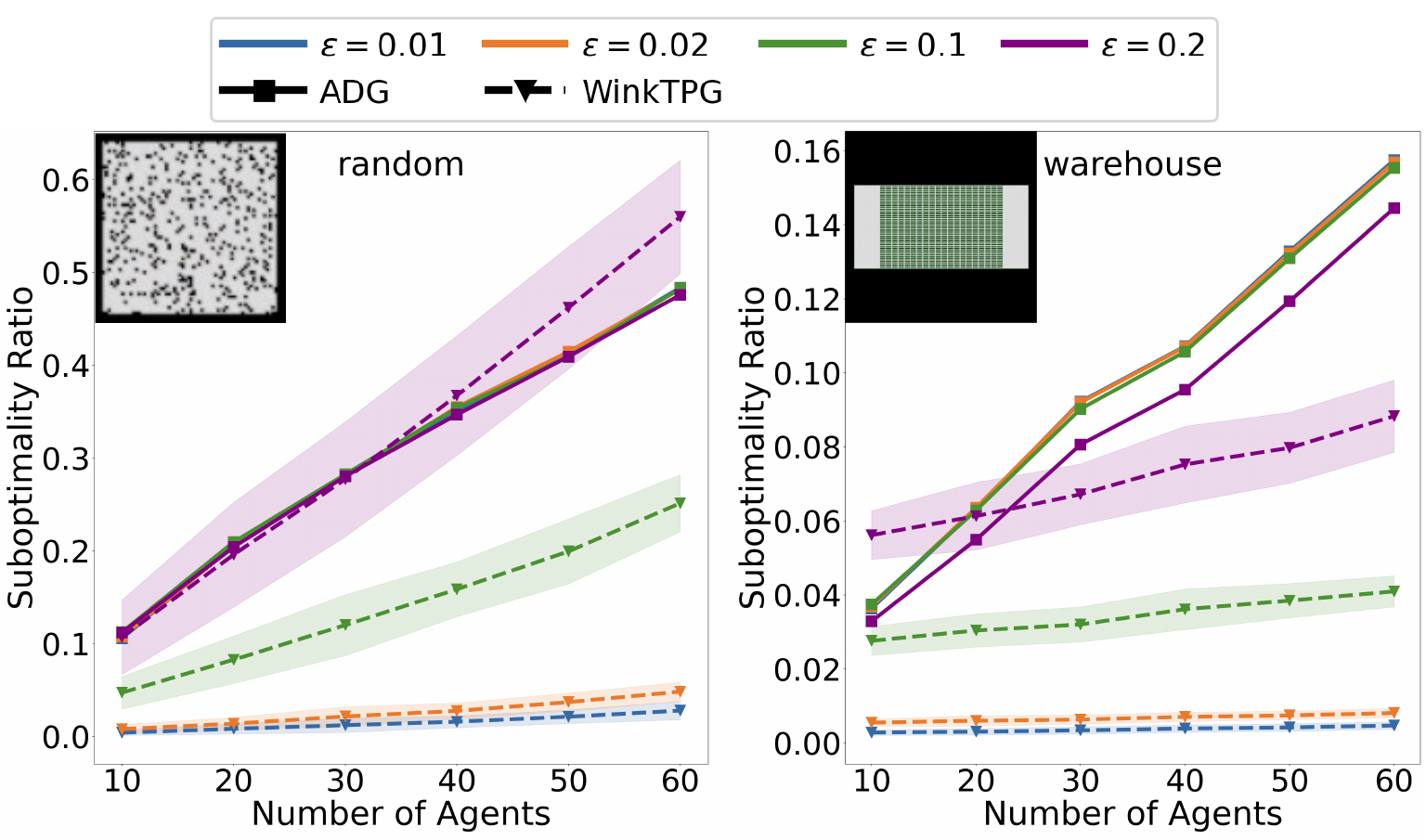}
    \caption{Comparison of solution quality between ADG (solid lines) and \mymethod (dashed lines) under the bounded uniform noise model across different timing bound $\epsilon$ (indicated by color).}
    \label{fig:stn_results}
\end{figure}

\noindent
\emph{Results under Bounded Uniform Noise Model.}
For the bounded uniform noise model, we vary the timing bound
$\epsilon \in \{0.01, 0.02, 0.1, 0.2\}\, s$.
\Cref{fig:stn_results} and \cref{table:additional_metrics} show the results. Again, \mymethod achieves deterministic precedence satisfaction for all $\epsilon$ values tested.
As expected, larger $\epsilon$ values lead to increased solution cost due to expanded safety intervals. Nevertheless, \mymethod maintains better coordination than ADG, particularly for moderate uncertainty levels (e.g., $\epsilon=0.02$ and $\epsilon=0.1$), where proactive interval reasoning avoids excessive waiting.

Overall, the results in this subsection show that \mymethod can effectively account for execution timing uncertainty and provide the intended robustness guarantees in practice. Under bounded uncertainty models, it achieves deterministic safety guarantees, while under the Gaussian noise model, the empirical violation rate remains below the theoretical bound. Moreover, \mymethod performs well under moderate uncertainty levels and exhibits a clear efficiency-robustness tradeoff as uncertainty increases. As discussed in \Cref{sec:ktpgu}, the appropriate timing uncertainty model depends on the underlying controller and operating environment. For example, bounded models are suitable for settings where execution timing errors remain within a bounded range.

\begin{figure}
    \centering
    \includegraphics[width=1.\linewidth]{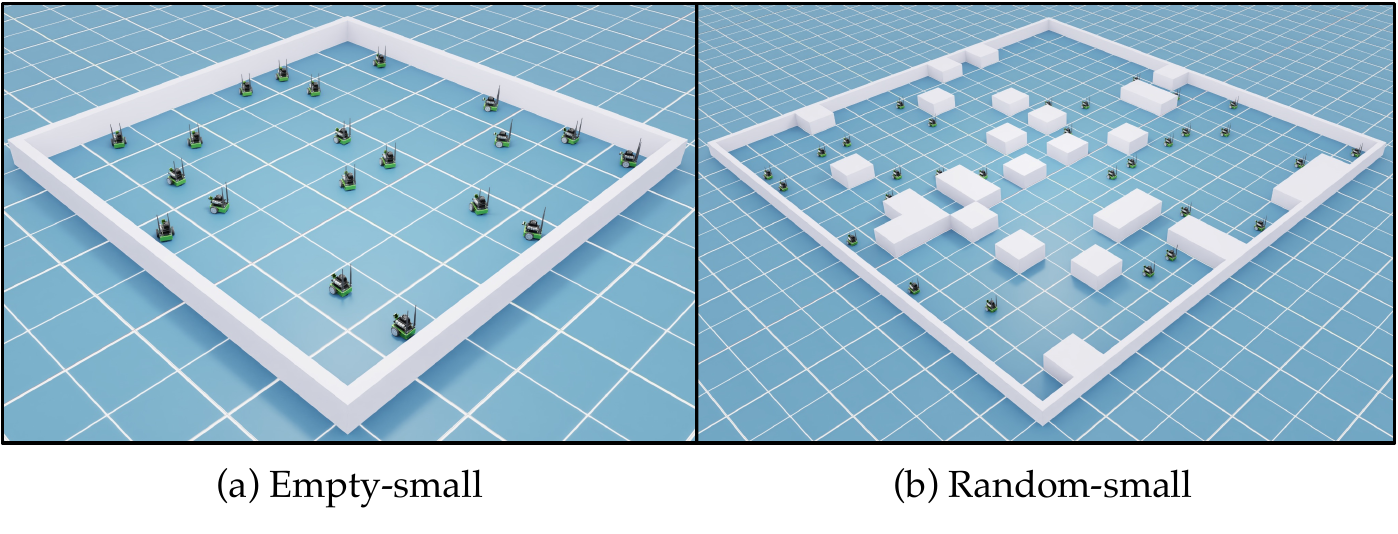}
    \caption{Simulation environments in Isaac-Sim.}
    \label{fig:isaac_sim}
\end{figure}

\subsection{Validation in Realistic Execution Settings}
\label{sec:isaac}

Finally, we validate \mymethod under high-fidelity physics-based simulation and real-world experiments. 
In both settings, we use \mymethod with the Gaussian noise model.

\subsubsection{Physics-Based Simulation}\label{sec:exp:isaac_sim}
We first validate \mymethod using NVIDIA Isaac Sim~\cite{nvidiaisaacsim}.
This setting explicitly models robot dynamics, control latency, and contact interactions, allowing us to assess whether the speed profiles produced by \uncertainPlanner can be reliably executed by realistic robot platforms.
We evaluate this setting on an Ubuntu 22.04 desktop machine with AMD 5800x processor and NVIDIA RTX 2080ti GPU.

\paragraph{Experiment Setup}
We evaluate our method in two simulated environments: \texttt{empty-small} (shown in~\cref{fig:isaac_sim}(a), size $8 \times 8$) and \texttt{random-small} (shown in~\cref{fig:isaac_sim}(b), size $16 \times 16$), both converted from 2-D grid maps. In both environments, each grid cell has a physical size of $0.5 \times 0.5\,\text{m}$. We use a differential-drive JetBot robot model with a diameter of $0.3\,\text{m}$. Each robot is subject to a linear speed limit of $[0, 1]\,\text{m/s}$, a linear acceleration limit of $[-0.5, 0.5]\,\text{m/s}^2$, and an angular velocity limit of $[0, \pi/4]\,\text{rad/s}$.
A simple feedforward–feedback velocity controller is used to execute speed profiles produced by \uncertainPlanner. At each control step, the feedforward term computes the reference linear and angular velocities corresponding to the current time from the planned speed profile, while the feedback term computes the pose-tracking error between the reference pose implied by the planned trajectory and the robot’s actual pose. The control command is given by
$u_i(t) = u_i^{\mathrm{ff}}(t) + K\, e_i(t)$,
where $u_i^{\mathrm{ff}}(t)$ denotes the feedforward velocity command, $e_i(t)$ denotes the pose-tracking error, and K is a feedback gain matrix. The resulting control commands are applied via Isaac Sim’s articulation interface.
To quantify tracking performance, we measure the timing error at the waypoints,
$\Delta t_i(q) = \bigl| t_i^{\text{sim}}(q) - t_i^{r}(q) \bigr|$,
where $t_i^{r}(q)$ is the planned arrival time determined by the speed profile and $t_i^{\text{sim}}(q)$ is the actual arrival time observed during execution.
Empirically, the mean timing error is around 0.25 s, with maximum errors around 0.5 s.
We use only simple parameter tuning for the controller; more advanced control
strategies could further reduce tracking error but are beyond the scope of this work.
We use 20 and 30 agents in the \texttt{empty-small} and \texttt{random-small} environments, respectively. For each environment, we evaluate 15 different MAPF plans generated by PBS under progressively increased uncertainty levels, with $\sigma \in \{0.0, 0.001, 0.005, 0.01, 0.05\}$.
A run is considered \emph{successful} if all agents reach their goals without any collision and no precedence violation happens. We report 
(1) the success rate over all trials,
(2) and the sum of execution time over successful runs.

\begin{figure}
    \centering
    \includegraphics[width=1.\linewidth]{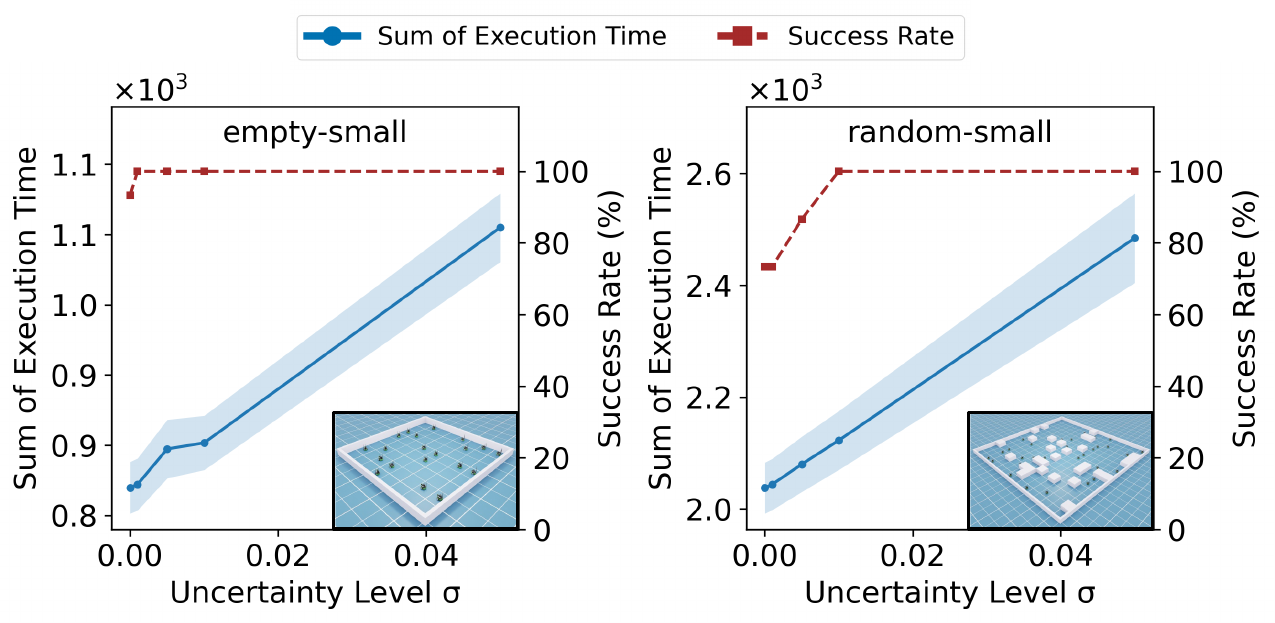}
    \caption{Sum of execution time and success rate as a function of uncertainty level $\sigma$ in physics-based simulation.}
    \label{fig:isaac_sim_exp}
\end{figure}

\paragraph{Experimental Results}
As shown in~\cref{fig:isaac_sim_exp}, in both environments the success rate increases monotonically with the modeled uncertainty level. Higher uncertainty induces more conservative speed profiles and larger safety margins, thereby improving robustness to execution variability in physics-based simulation.
This improved robustness comes at the cost of solution quality: as the uncertainty level increases, execution time grows noticeably, revealing a fundamental trade-off between robustness and efficiency. Consequently, the uncertainty level should be selected based on the target application and controller characteristics. Larger margins are appropriate for safety-critical scenarios or controllers with high variance, whereas smaller margins may suffice when higher efficiency is desired and execution can be tracked accurately.

\begin{figure}
    \centering
    \includegraphics[width=1.\linewidth]{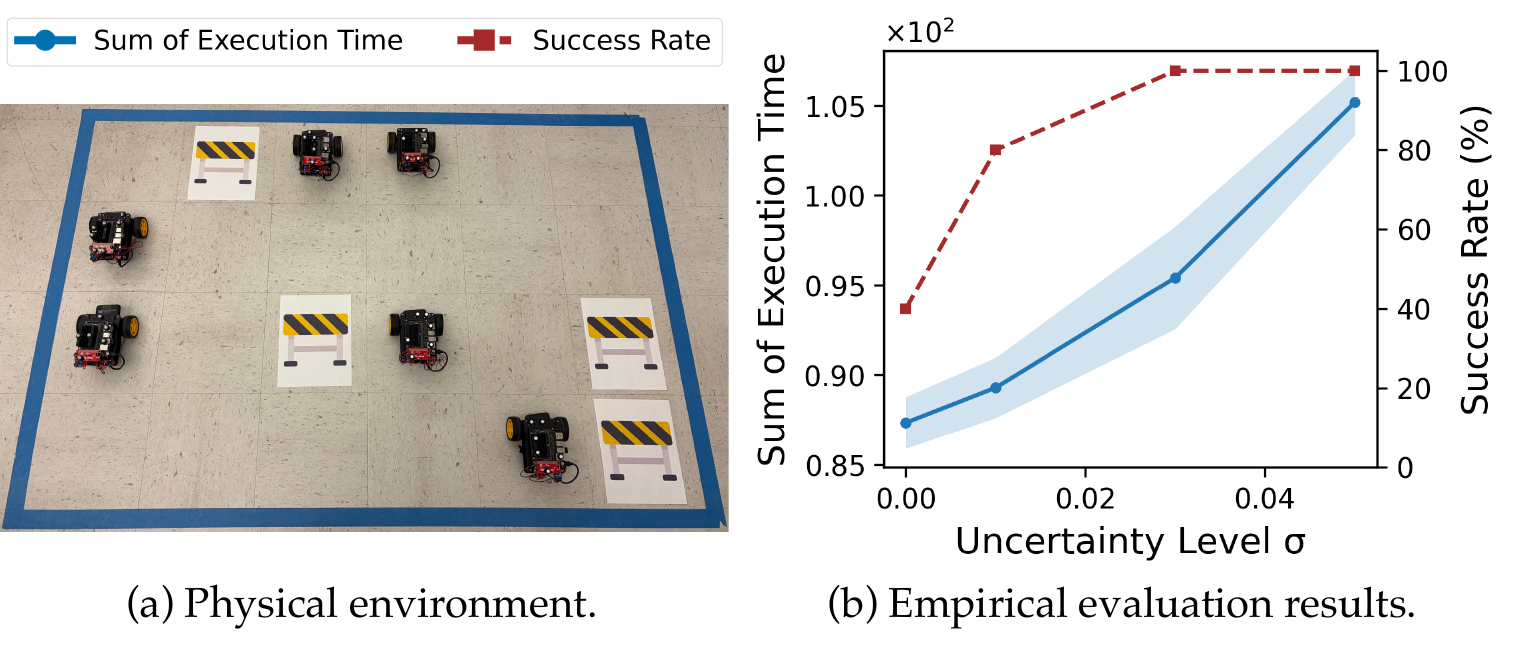}
    \caption{Real-world robot experiments. The red line denotes the success rate, and the blue line denotes the sum of execution time.}
    \label{fig:realworld-exp}
\end{figure}

\subsubsection{Physical Robot Experiments}
Finally, we demonstrate the execution of \mymethod on real-world mobile robots.

\paragraph{Experiment Setup}
We deploy \mymethod on a team of 6 JetBot robots, consistent with the setup described in~\Cref{sec:exp:isaac_sim}, operating in a structured indoor environment as shown in~\Cref{fig:realworld-exp}. 
The workspace is configured as a $4\times6$ grid, where each cell has a size of $0.3\,\text{m} \times 0.3\,\text{m}$. Both obstacles and workspace boundaries are physically defined.
Robots communicate with a centralized coordination server over a local wireless network.
A Vicon motion capture system provides real-time localization for all robots during execution. Similarly, we evaluate 5 different MAPF plans generated by PBS with progressively increasing uncertainty levels, with $\sigma \in \{0.0, 0.01, 0.02, 0.05\}$.

\paragraph{Experimental Results}
As shown in~\Cref{fig:realworld-exp}, \mymethod executes reliably on physical robots under realistic sensing, actuation, and communication conditions. 
Increasing $\sigma$ leads to larger safety margins and more conservative speed profiles, resulting in longer execution times but improved reliability. 
The success rate stabilizes at 100\% once sufficient margin is introduced, confirming that the proposed safety mechanism effectively compensates for execution uncertainty.
These results demonstrate that the framework transfers from simulation to physical robots and exhibits the expected robustness-efficiency trade-off in practice.

\section{Conclusion}
In this paper, we introduced \mymethod, an execution framework for efficiently and robustly executing MAPF plans in environments with execution timing uncertainty.
\mymethod uses our proposed kinodynamic planning method \planner to generate collision-free speed profiles for all agents. It incorporates location updates from agents to replan dynamically, while providing deterministic guarantees under bounded delays and probabilistic guarantees under stochastic models.
Empirically, we evaluated our methods in the MAPF domain, showing up to 51.7\% improvement in solution quality compared to existing methods, while generating speed profiles for up to 1,000 agents in 1 second.
Under multiple uncertainty models, \mymethod achieved the intended robustness guarantees in practice. The results also showed that \mymethod outperforms ADG under moderate uncertainty.
We further validated the practical applicability of \mymethod through experiments in a high-fidelity physics-based simulator and real-world mobile robots, confirming that the generated speed profiles are physically executable and robust to execution variability.
Future work includes exploring the integration of \mymethod with advanced MAPF replanning techniques such as rolling-horizon conflict resolution~\cite{li2021lifelong} or switchable TPG~\cite{feng2024real}.
Another important direction is to extend the uncertainty analysis beyond the current setting. This includes heavy-tailed disturbances, correlated delays arising from shared infrastructure, and adversarial failures, none of which are explicitly modeled in the present framework.
Finally, extending \mymethod to incorporate decentralized planning and execution mechanisms could further enhance its scalability and robustness in highly dynamic and large-scale scenarios.

 

\bibliographystyle{IEEEtran}
\bibliography{reference}

@string{ICAPS = {Proceedings of the International Conference on Automated Planning and Scheduling}}

@string{AAAI = {Proceedings of the AAAI Conference on Artificial Intelligence}}

@string{ICRA = {Proceedings of the IEEE International Conference on Robotics and Automation}}

@string{IROS = {Proceedings of the IEEE International Conference on Intelligent Robots and Systems}}

@string{SoCS = {Proceedings of the International Symposium on Combinatorial Search}}

@string{RAL = {IEEE Robotics and Automation Letters}}

@article{wen2022cl,
  title={{CL-MAPF}: Multi-agent path finding for car-like robots with kinematic and spatiotemporal constraints},
  author={Wen, Licheng and Liu, Yong and Li, Hongliang},
  journal={Robotics and Autonomous Systems},
  volume={150},
  pages={103997},
  year={2022},
  publisher={Elsevier}
}

@inproceedings{saccon2022comparing,
  title={Comparing multi-agent path finding algorithms in a real industrial scenario},
  author={Saccon, Enrico and Palopoli, Luigi and Roveri, Marco},
  booktitle={Proceedings of the International Conference of the Italian Association for Artificial Intelligence},
  pages={184--197},
  year={2022}
}

@article{van2011lqg,
  title={{LQG-MP}: Optimized path planning for robots with motion uncertainty and imperfect state information},
  author={Van Den Berg, Jur and Abbeel, Pieter and Goldberg, Ken},
  journal={The International Journal of Robotics Research},
  volume={30},
  number={7},
  pages={895--913},
  year={2011},
}

@article{yan2025advancing,
  title={Advancing {MAPF} toward the real world: A scalable multi-agent realistic testbed ({SMART})},
  author={Yan, Jingtian and Li, Zhifei and Kang, William and Zheng, Kevin and Zhang, Yulun and Chen, Zhe and Zhang, Yue and Harabor, Daniel and Smith, Stephen F and Li, Jiaoyang},
  journal=RAL,
  year={2026},
  note={in press}
}

@inproceedings{lehoux2024multi,
  title={Multi-agent path finding with real robot dynamics and interdependent tasks for automated warehouses},
  author={Lehoux-Lebacque, Vassilissa and Silander, Tomi and Loiodice, Christelle and Lee, Seungjoon and Wang, Albert and Michel, Sofia},
  booktitle={Proceedings of the European Conference on Artificial Intelligence},
  year={2024},
  pages={4393--4401}
}

@article{bartak2019multi,
  title={Multi-Agent Path Finding on Real Robots},
  author={Bart{\'a}k, Roman and {\v{S}}vancara, Ji{\\v{r}}{\\'\\i} and Skopkova, Veronika and Nohejl, David and Krasicenko, Ivan},
  journal={AI Communications},
  volume={32},
  number={3},
  pages={175--189},
  year={2019}
}

@inproceedings{honig2022db,
  title={db-{A}*: Discontinuity-bounded search for kinodynamic mobile robot motion planning},
  author={H{\"o}nig, Wolfgang and Ortiz-Haro, Joaquim and Toussaint, Marc},
  booktitle=IROS,
  pages={13540--13547},
  year={2022}
}

@article{lipp2014minimum,
  title={Minimum-time speed optimisation over a fixed path},
  author={Lipp, Thomas and Boyd, Stephen},
  journal={International Journal of Control},
  volume={87},
  number={6},
  pages={1297--1311},
  year={2014},
  publisher={Taylor \& Francis}
}

@inproceedings{ma2017feasibility,
  title={Feasibility study: {M}oving non-homogeneous teams in congested video game environments},
  author={Ma, Hang and Yang, Jingxing and Cohen, Liron and Kumar, TK and Koenig, Sven},
  booktitle={Proceedings of the AAAI Conference on Artificial Intelligence and Interactive Digital Entertainment},
  volume={13},
  number={1},
  pages={270--272},
  year={2017}
}

@article{berndt2023receding,
  title={Receding horizon re-ordering of multi-agent execution schedules},
  author={Berndt, Alexander and Van Duijkeren, Niels and Palmieri, Luigi and Kleiner, Alexander and Keviczky, Tam{\'a}s},
  journal={IEEE Transactions on Robotics},
  volume={40},
  pages={1356--1372},
  year={2023}
}

@misc{nvidiaisaacsim,
  author = {NVIDIA},
  title = {{I}saac {S}im},
  year = {2023},
  howpublished = {\url{https://developer.nvidia.com/isaac-sim}},
  note = {Accessed: 2025-07-10}
}

@inproceedings{zhang2021temporal,
  title={Temporal reasoning with kinodynamic networks},
  author={Zhang, Han and Tiruviluamala, Neelesh and Koenig, Sven and Kumar, TK Satish},
  booktitle=ICAPS,
  volume={31},
  pages={415--425},
  year={2021}
}

@inproceedings{li2023intersection,
  title={Intersection coordination with priority-based search for autonomous vehicles},
  author={Jiaoyang Li and The Anh Hoang and Eugene Lin and Hai L. Vu and Sven Koenig},
  booktitle=AAAI,
  volume={37},
  pages={11578--11585},
  year={2023}
}

@inproceedings{ma2019searching,
  title={Searching with consistent prioritization for multi-agent path finding},
  author={Ma, Hang and Harabor, Daniel and Stuckey, Peter J and Li, Jiaoyang and Koenig, Sven},
  booktitle=AAAI,
  volume={33},
  pages={7643--7650},
  year={2019}
}

@inproceedings{ali2023safe,
  title={Safe Interval Path Planning with Kinodynamic Constraints},
  author={Ali, Zain Alabedeen and Yakovlev, Konstantin},
  booktitle=AAAI,
  volume={37},
  pages={12330--12337},
  year={2023}
}

@inproceedings{honig2016multi,
  title={Multi-agent path finding with kinematic constraints},
  author={H{\"o}nig, Wolfgang and Kumar, T. K. Satish and Cohen, Liron and Ma, Hang and Xu, Hong and Ayanian, Nora and Koenig, Sven},
  booktitle=ICAPS,
  volume={26},
  pages={477--485},
  year={2016}
}

@inproceedings{Stern2019benchmark,
  author    = {Roni Stern and Nathan R. Sturtevant and Ariel Felner and Sven Koenig and Hang Ma and Thayne T. Walker and Jiaoyang Li and Dor Atzmon and Liron Cohen and T. K. Satish Kumar and Eli Boyarski and Roman Bart{\'{a}}k},
  title     = {Multi-Agent Pathfinding: Definitions, Variants, and Benchmarks},
  booktitle = SoCS,
  pages     = {151--159},
  year      = {2019}
}

@ARTICLE{honig2019warehouse,
  author={H\"{o}nig, Wolfgang and Kiesel, Scott and Tinka, Andrew and Durham, Joseph W. and Ayanian, Nora},
  journal=RAL, 
  title={Persistent and robust execution of {MAPF} schedules in warehouses}, 
  year={2019},
  volume={4},
  number={2},
  pages={1125--1131}}

@inproceedings{li2021eecbs,
  title={{EECBS}: A bounded-suboptimal search for multi-agent path finding},
  author={Li, Jiaoyang and Ruml, Wheeler and Koenig, Sven},
  booktitle=AAAI,
  volume={35},
  pages={12353--12362},
  year={2021}
}

@inproceedings{li2021lifelong,
  title={Lifelong multi-agent path finding in large-scale warehouses},
  author={Li, Jiaoyang and Tinka, Andrew and Kiesel, Scott and Durham, Joseph W and Kumar, TK Satish and Koenig, Sven},
  booktitle=AAAI,
  volume={35},
  pages={11272--11281},
  year={2021}
}

@ARTICLE{yan2024PSB,
  author={Yan, Jingtian and Li, Jiaoyang},
  journal=RAL, 
  title={Multi-Agent Motion Planning With Bézier Curve Optimization Under Kinodynamic Constraints}, 
  year={2024},
  volume={9},
  number={3},
  pages={3021--3028}}

@article{okumura2022priority,
  title={Priority inheritance with backtracking for iterative multi-agent path finding},
  author={Okumura, Keisuke and Machida, Manao and D{\'e}fago, Xavier and Tamura, Yasumasa},
  journal={Artificial Intelligence},
  volume={310},
  pages={103752},
  year={2022}
}

@inproceedings{moldagalieva2024db,
  title={db-{CBS}: {D}iscontinuity-bounded conflict-based search for multi-robot kinodynamic motion planning},
  author={Moldagalieva, Akmaral and Ortiz-Haro, Joaquim and Toussaint, Marc and H{\"o}nig, Wolfgang},
  booktitle=ICRA,
  pages={14569--14575},
  year={2024}
}

@inproceedings{andreychuk2021improving,
  title={Improving continuous-time conflict based search},
  author={Andreychuk, Anton and Yakovlev, Konstantin and Boyarski, Eli and Stern, Roni},
  booktitle=AAAI,
  volume={35},
  pages={11220--11227},
  year={2021}
}

@article{peltzer2020stt,
  title={{STT-CBS}: {A} conflict-based search algorithm for multi-agent path finding with stochastic travel times},
  author={Peltzer, Oriana and Brown, Kyle and Schwager, Mac and Kochenderfer, Mykel J and Sehr, Martin},
  journal={arXiv preprint arXiv:2004.08025},
  year={2020}
}

@article{atzmon2020robust,
  title={Robust multi-agent path finding and executing},
  author={Atzmon, Dor and Stern, Roni and Felner, Ariel and Wagner, Glenn and Bart{\'a}k, Roman and Zhou, Neng-Fa},
  journal={Journal of Artificial Intelligence Research},
  volume={67},
  pages={549--579},
  year={2020}
}

@inproceedings{morris2016planning,
  title={Planning, scheduling and monitoring for airport surface operations},
  author={Morris, Robert and Pasareanu, Corina S and Luckow, Kasper and Malik, Waqar and Ma, Hang and Kumar, TK Satish and Koenig, Sven},
  booktitle={AAAI Workshop: Planning for Hybrid Systems},
  year={2016}
}

@article{wurman2008coordinating,
  title={Coordinating hundreds of cooperative, autonomous vehicles in warehouses},
  author={Wurman, Peter R and D'Andrea, Raffaello and Mountz, Mick},
  journal={AI magazine},
  volume={29},
  number={1},
  pages={9--19},
  year={2008}
}

@inproceedings{feng2024real,
  title={A real-time rescheduling algorithm for multi-robot plan execution},
  author={Feng, Ying and Paul, Adittyo and Chen, Zhe and Li, Jiaoyang},
  booktitle=ICAPS,
  volume={34},
  pages={201--209},
  year={2024}
}

@article{dechter1991temporal,
  title={Temporal constraint networks},
  author={Dechter, Rina and Meiri, Itay and Pearl, Judea},
  journal={Artificial intelligence},
  volume={49},
  number={1-3},
  pages={61--95},
  year={1991}
}

@book{casella2024statistical,
  title={Statistical inference},
  author={Casella, George and Berger, Roger},
  year={2024},
  publisher={CRC press}
}

@inproceedings{liu2017speed,
  title={Speed profile planning in dynamic environments via temporal optimization},
  author={Liu, Changliu and Zhan, Wei and Tomizuka, Masayoshi},
  booktitle={Proceedings of the IEEE Intelligent Vehicles Symposium},
  pages={154--159},
  year={2017}
}

\begin{IEEEbiography}[{\includegraphics[width=1in]{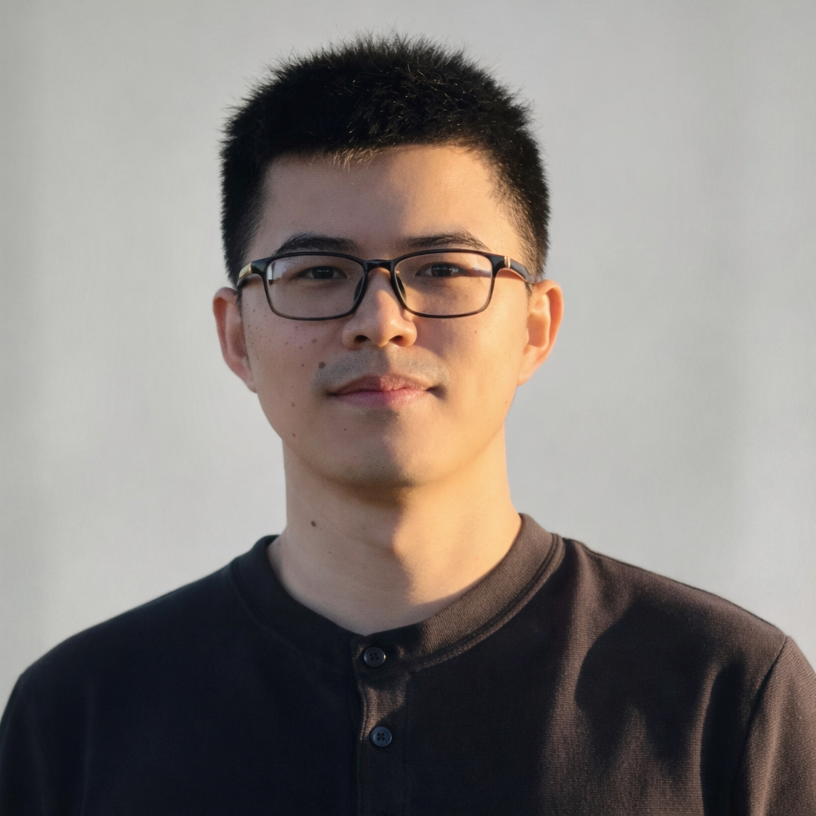}}]{Jingtian Yan}
(Graduate Student Member, IEEE) received the B.Eng. degrees from Zhejiang University, China, and the M.S. degree from Carnegie Mellon
University, Pittsburgh, PA, USA, where he is currently pursuing the Ph.D. degree. His research interests lie in Multi-Agent Path Finding, robotic execution under uncertainty, and large-scale multi-robot systems.
\end{IEEEbiography}

\begin{IEEEbiography}[{\includegraphics[width=1in]{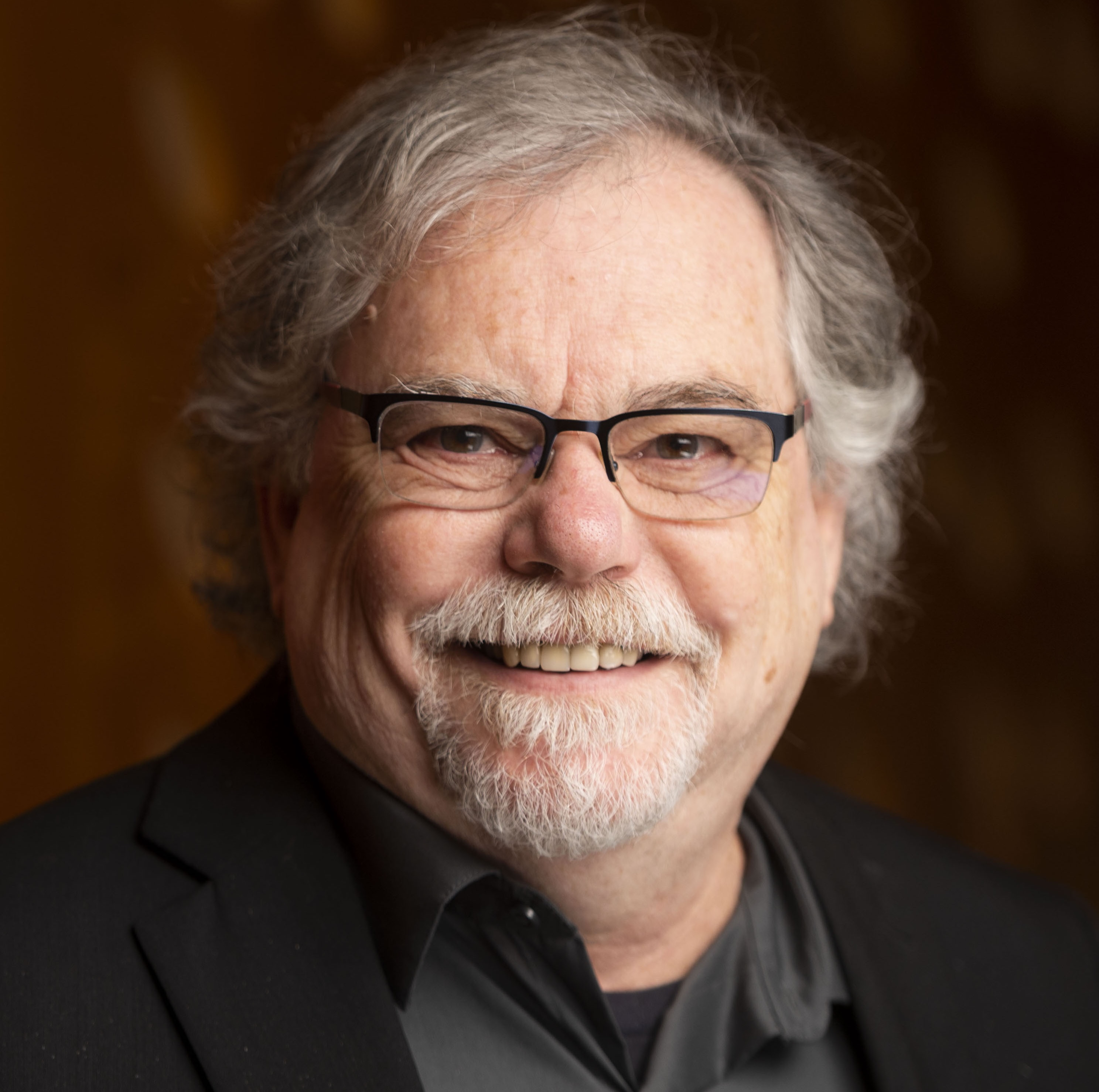}}]{Stephen F. Smith} (Member, IEEE) received the Ph.D. degree in Computer Science from the University of Pittsburgh in 1980. He joined the faculty of the Robotics Institute at Carnegie Mellon University in 1982, where he is currently a Research Professor. His research focuses broadly on the theory and practice of next generation systems for automated planning, scheduling and control of large multi-actor systems.
\end{IEEEbiography}

\begin{IEEEbiography}[{\includegraphics[width=1in]{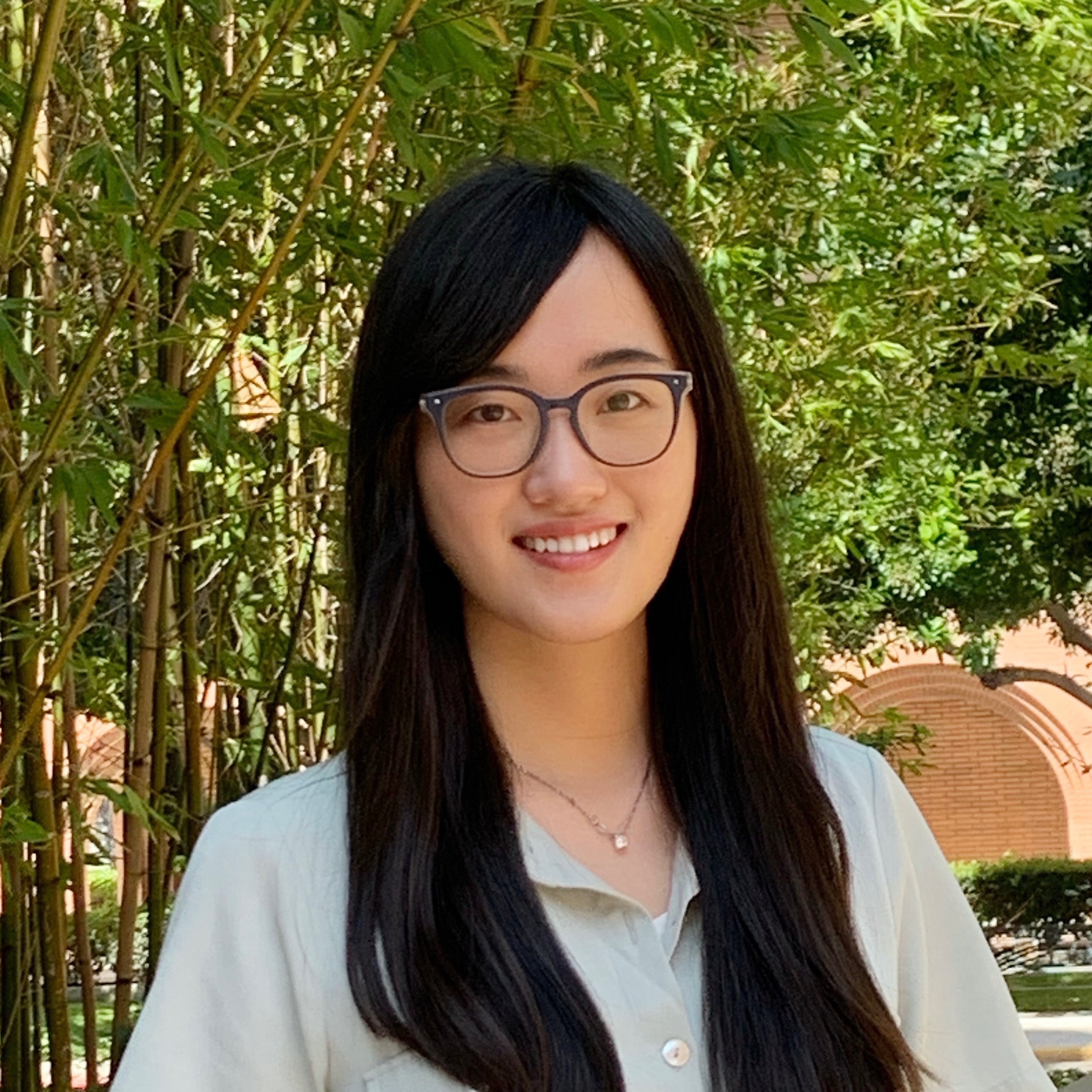}}]{Jiaoyang Li} (Member, IEEE) received the Ph.D. degree in Computer Science from the University of Southern California, USA, in 2022 and the B.Eng. degree in Automation from Tsinghua University, China, in 2017. She is currently an Assistant Professor in the Robotics Institute of Carnegie Mellon University, USA. Her research interests lie in multi-robot planning and coordination.
\end{IEEEbiography}

\end{document}